\documentclass[11pt]{imsart}

\usepackage{graphicx}
\usepackage{amssymb}
\usepackage{amsmath}
\usepackage{amsthm}
\usepackage{mathtools}
\usepackage{caption}
\usepackage{subcaption}
\usepackage[ruled,vlined]{algorithm2e}
\usepackage{float}
\newtheorem{theorem}{Theorem}
\newtheorem{lemma}[theorem]{Lemma}
\newtheorem{proposition}[theorem]{Proposition}
\newtheorem{corollary}[theorem]{Corollary}
\newtheorem{definition}[theorem]{Definition}
\newtheorem{remark}[theorem]{Remark}

\let\oldproofname=\proofname
\renewcommand{\proofname}{\rm\bf{\oldproofname}}

\startlocaldefs
\endlocaldefs

\graphicspath{ {images/} }

\begin{document}

\begin{frontmatter}

\title{A Critical Connectivity Radius for Segmenting Randomly-Generated, High Dimensional Data Points}
\runtitle{Critical Connectivity Radius}

\author{\fnms{Robert A.} \snm{Murphy, Ph.D.}\ead[label=e1]{robert.a.murphy@wustl.edu}}
\address{\printead{e1}}

\runauthor{Murphy}

\begin{abstract}
Motivated by a $2$-dimensional (unsupervised) image segmentation task whereby local regions of pixels are clustered via edge detection methods, a more general probabilistic mathematical framework is devised.  Critical thresholds are calculated that indicate strong correlation between randomly-generated, high dimensional data points that have been projected into structures in a partition of a bounded, $2$-dimensional area, of which, an image is a special case.  A neighbor concept for structures in the partition is defined and a critical radius is uncovered.  Measured from a central structure in localized regions of the partition, the radius indicates strong, long and short range correlation in the count of occupied structures.  The size of a short interval of radii is estimated upon which the transition from short-to-long range correlation is virtually assured, which defines a demarcation of when an image ceases to be "interesting".
\end{abstract}



\end{frontmatter}

\newpage

\tableofcontents

\newpage


\section{Introduction and Related Works}

A $2$-dimensional image can be conceptualized as a special case of a random process that generates a set of high dimensional (integer) data projected to a bounded, uniformly partitioned, $2$-dimensional subset of $\mathbb{R}^{2}$.  A \textit{stationary} noise process, applied after the projection operator, adds blurring and other degradation effects, which further complicates the task of segmenting correlated integer values (termed pixel intensities), following a loss of higher dimensional information during the projection.  Indeed, in the presence of noise, optimally-sized, localized regions of pixel intensities are heavily dependent upon the range of the applied blurring and degradation effects, measured (in pixels) from a central pixel in the region.  The expected tendency is that, beyond a certain critical number of pixels, long range noise effects, are mostly uncorrelated, resulting in an "interesting" image consisting of multiple regions of pixels to be segmented into clusters of individual objects.  Collectively, these segments are an example of a \textit{graph property} of correlated structures of pixel intensities, providing the beginnings of a framework for measuring events on segmented objects, as a function of a range, measured from the center of each localized region.

Suppose infinitely many copies of a bounded structure are used to partition $\mathbb{R}^{2}$ and let $\mathcal{B} \subset \mathbb{R}^{2}$ be a bounded subset containing finitely many copies of the bounded structure.  Further suppose that structures in the partition are \textit{neighbors}, if their respective boundaries have non-empty intersection.  Infinitely many of the bounded structures in the partition are individually occupied by exactly one point, at the center of the structure, independently of all other structures.

In $\cite{Grimmett,Grimmett2}$, it is shown that, if the probability of neighboring structures, each containing related points, is greater than some critical value, then with probability $1$, a path can be traced from any starting occupied bounded structure to any ending occupied bounded structure.  The path in between the start and the end consists entirely of neighboring occupied structures.  From this statement, the contrapositive statement is obtained such that, if the probability of neighboring, occupied structures is less than or equal to the same critical value, then with probability $1$, no such path exists for any two bounded structures.  Hence, all points are either related to no other points or only related to finitely many points in neighboring bounded structures.  It is this contrapositive condition that is of interest, as clusters are formed by groupings of inter-related data points.

Rarely does inter-related, real-world data conform to a predefined, rigid partition, as described above.  As such, after removing the rigid partition of $\mathbb{R}^{2}$, suppose the data are modeled by a point process which randomly generates data within $\mathcal{B}$ according to some predetermined probability distribution.  Data in $\mathcal{B}$ are inter-related, if they are within a certain distance of one another or some common center, such as the average of a set of previously-grouped, inter-related data points.

In $\cite{Meester}$, it is shown that, with probability $1$, there is a path connecting any two points in $\mathcal{B}$, if the density of data relative to the area of $\mathcal{B}$ is beyond a certain critical number.  Alternatively, if the maximum distance between inter-related points is larger than a certain critical number, then a similar connected path exists in $\mathcal{B}$.

In $\cite{Murphy}$, it is proven that an ordered set of data, which is assumed to be spatially uniformly distributed, will form clusters, if the measured distances between data points (or some common data point in each cluster) are below a certain threshold.  The threshold is computed as a function of the number of data points sampled from the total population of data.

To apply the results from $\cite{Murphy}$, we can make the assumption that the point process generates points according to the normal distribution by making use of a theorem from $\cite{Hogg}$ in probability called the Central Limit Theorem.  In essence, this theorem states that any set of randomly distributed data with finite mean and variance will tend to be normally distributed as the sample size grows large.  This accounts for the ubiquity of the normal distribution in nature and why it is safe to make an assumption of normality in most cases.  As such, the point process is allowed to run until $n = M^{2}$ points are generated, as represented by a sequence of independent, identically distributed random variables, $X_{1}$, $X_{2}$, ... , $X_{n}$, each with mean $0$ and variance $1$.

An order statistic can be applied to a collection of individual observations from these random variables so that $X_{k_{1}} < X_{k_{2}} < ... < X_{k_{n}}$, where $\sigma(i) = k_{i}$ for $i \in \{1,2,...,n\}$ is a permutation resulting in the ordering of $\{X_{i}\}_{i = 1}^{n}$.  Note that, by default, $\{X_{k_{i}}\}_{i = 1}^{n}$ is a sequence of \textit{dependent} random variables since for each $i \in \{2,3,...,n\}$, the random variable $X_{k_{i}}$ depends upon $X_{k_{j}}$ for all $j < i$.  The edge space of the higher dimensional observations of the point process are then embedded within the $2$-dimensional plane to aid in the calculation of a threshold on the distance measure.  This threshold is shown to prevent the computational expense of ordering the normally distributed observations.

To make the assumption of uniformity of the ordered set of points and to perform segmenting therein, it is first noted that the Beta distribution is the probability distribution of an order statistic applied to observations from normally distributed random variables $\cite{Mameli}$.  The Beta distribution shape parameters are then defined to be $\alpha = 1 = \beta$ since we will want to estimate a lower bound on the expected number of segmented regions $K$ obtained from the $n$ outputs generated by a point process.  A threshold is estimated from $K$ and $n$ to define the size of partitions applied to $\mathcal{B} \subset \mathbb{R}^{2}$, where an a'priori estimate of $K$ is formalized mathematically.  Then, each data point is injectively generated into exactly $1$ structure of the partition of $\mathcal{B}$, with the $n$ samples uniformly distributed among $K$ disjoint, segmented regions.  Hence, uniformity of samples from the point process will be mathematically formalized.

Without loss of generality, the bounded region can be assumed to be of unit area, centered at the origin.  From $\cite{Grimmett,Grimmett2,Kesten}$, the shape of individual structures in the partition allows for the calculation of $K$ as a function of $n$ and a constant that depends on the shape of the individual structures.

Finally, with the defined partition, it is shown that under certain conditions, no approximation of probabilities in the continuum is required to prove the existence of a path of any order, as in $\cite{Meester}$.  Instead, the probability of a long range path in the continuum is equivalent to the probability of a long range path in the presence of the defined posterior partition, when the threshold defines the size of structures.  On a bounded interval containing the threshold, the probability of the existence of a long range path of connected points rises sharply.  Lastly, the probability measure in question is found to be a unique random cluster measure which realizes a set of conditional probability measures.  As such, the point process samples from the collection of conditional probability measures to form clusters, when points connect at a distance less than or equal to the threshold.

In $\cite{Cai}$, Cai, et.al. investigate the problem of partial connectivity of randomly distributed points in a bounded region by making the assumption that, relative to the size of the bounded region, the number of points to be generated is relatively small.  As such, a Poisson-distributed point process generates an independent set of data in the designated region.  Copies of a hexagon of some fixed, immutable size, which is not dependent upon distances between generated points, are used to partition the bounded region.  Points in the region are deemed to \textit{connect} to form an \textit{open} edge, if after the region is partitioned, the points lie within the same hexagon or neighboring hexagons, where hexagons are \textit{neighbors}, if their respective boundaries have non-empty intersection.  Otherwise, the edge between two points is \textit{closed}.  Likewise, they define the logical points at the centers of two neighboring hexagons to \textit{connect} to form an \textit{open} edge, if each neighboring hexagon independently contains at least one of the generated points.

Also in $\cite{Cai}$, Cai, et.al. compute probabilities as a function of the density of hexagons which are occupied by at least one point.  They showed that if the number of hexagons in the fixed partition is unbounded and the number of points generated within the continuum of the bounded region is below (or at) a critical threshold, then the probability of a majority of the occupied hexagons being connected tends to zero.  On the other hand, if the density of occupied hexagons is within a short interval around the critical threshold, then a connecting path of hexagons occurs with probability that rises sharply.

$\cite[Thm.\ (1.1)]{Goel}$ gives an estimate of the length of the short interval.  If the area of the bounded region is assumed to be one, without loss of generality, then the estimate of the length of the short interval can never be any better than $c/\sqrt{n\log{n}}$.

The distance notion of connectivity will be used (without the presence of a partition), the same as in $\cite{Cai}$.  However, the ideas presented here diverge from $\cite{Cai}$ in that the prototypical hexagon used in the defined posterior partition of the bounded region is allowed to change in size.  Then, the logical centers are closer to (or further away from) each other and point density is inversely proportional to the maximum connection length.  Therein lies an added advantage when calculating the probability of a connected path of hexagons or points.  Moreover, if the prototypical hexagon always shrinks as the point process generates more points, then it should be expected that the critical threshold and the length of the short interval are adjusted accordingly.

In sec. $(\ref{proc})$, we describe the procedure for estimating a critical radius of connectivity and its sharp threshold interval.  Two different formulations are derived, one in the continuum [sec. $(\ref{rgg})$], without a partition of bounded regions, and the other in the presence of a partition [sec. $(\ref{hpm})$].  We show that the critical radii in both formulations are unique and that under certain conditions, both formulations are equivalent.

In sec. $(\ref{rgg})$, we formulate the notion of connectivity of randomly-generated data points in the continuum using random geometric graphs and prove certain continuity results of the probability measure of a class of graph properties.  We use the continuity results to show the existence of a critical connectivity radius (equivalent density of data points) and prove that the length of the sharp threshold interval containing the critical radius is of a certain size, depending upon the number of data points and the length of the critical radius.

In sec. $(\ref{hpm})$, we uniformly partition the bounded region into shapes of the same size and formulate the notion of connectivity of randomly-generated data points using the random cluster model.  The continuity results of sec. $(\ref{rgg})$ still hold true and are used to show the existence of a critical connectivity radius, with the associated sharp threshold interval length being of a certain size, depending upon the number of data points and the length of the critical radius.  In addition, relationships between the graph properties and probabilities of the graph properties are proven, along with a result about the relationship between the critical radii.  These results, as well as other results from the random cluster model, are used to give a practical estimate of the length of the sharp threshold interval and a lower bound estimate of the change in the probability of the class of graph properties.  Finally, it is shown that under certain conditions, the probabilities of the graph properties are equivalent and the critical radii are of the same length under both formulations.

\section{Procedure}
\label{proc}

In the special case of image segmentation, the dual goal is to segregate pixels belonging to distinct objects from all other objects and to demarcate the background as a distinct object.  Separate objects delineate "interesting" features inherent in the image space.  One method of segmentation is open/closed edge detection whereby transitions can demarcate boundary elements between distinct objects.

Pixel intensity is represented by measurable functions of $8$-bit integer values assigned to uniformly sized structures (pixels) in a partition of the image space.  One way to define open/closed edges between neighboring structures in a partition is to say that an edge is \textit{open} if the respective intensities are the same and an edge is \textit{closed} otherwise.  As such, we want to find a small radius (measured in a count of structures from a common center) that defines the smallest areas consisting of open edges.  Then, transitions between boundaries of the small structures allow for the detection of closed edges which may indicate separate objects during the segmentation.  We will show the existence of a critical radius that accomplishes this part of the segmentation task and show that it is unique.  In addition, we will show how to estimate the critical radius and the length of its sharp threshold interval.

Let $n$ be a Poisson random variable which takes a particular value (again denoted as $n$), with density parameter $\lambda = \lambda(n)$ indicating the expected number of points generated per unit area of a bounded region, $\mathcal{B}$.  Let $X$ be a point process such that $X(\mathcal{B}) = \mathcal{X}_{n}$ is a set of $n$ points which are spatially uniformly distributed throughout $\mathcal{B}$.  For some fixed $r > 0$, points in $\mathcal{X}_{n}$ connect if their mutual distance is within $r$ and we say that the points are \textit{$r$-connected}.  We are interested in the graph property $\mathcal{A}_{[n,\rho]}^{r}$ consisting of all $r$-connected paths through $\hat{0} \in \mathcal{B}$ that contain at least half of all generated points in $\mathcal{X}_{n} \subset \mathcal{B}$.

Points in $\mathcal{X}_{n}$ are spatially uniformly distributed and the measure of $r$-connected points in $\mathcal{X}_{n}$ is given by $\rho \in (\frac{1}{2},1)$.  The interval bounding $\rho$ comes from an important result from $\cite{Kesten}$ indicating that the critical probability of connection on the square lattice is $\frac{1}{2}$, which is paramount in the calculation of the critical probability obtained by finding a radius $r_{0} > 0$ such that $\mathbf{P}(\mathcal{A}_{[n,\rho]}^{r_{0}}) = \frac{1}{2}$.  Thus, let $\epsilon > 0$ be given and let $r = r(n,\rho,\epsilon)$ be the least connectivity radius $r > 0$ such that $\mathbf{P}(\mathcal{A}_{[n,\rho]}^{r}) \geq \epsilon$.  It will be shown that $\mathbf{P}(\mathcal{A}_{[n,\rho]}^{r})$ is an increasing function of the connection radius $r$, with the goal of estimating the length of the interval of connectivity radii such that the occurrence of $\mathcal{A}_{[n,\rho]}^{r}$ increases in probability on $[\epsilon,1-\epsilon]$.  In this analysis, we use the square lattice and hexagonal lattice interchangeably since any hexagon can be uniquely inscribed inside a square, up to a rotation.

As an integral step in estimating the length of the interval of radii, continuity in $r$ and $\rho$ of $\mathbf{P}(\mathcal{A}_{[n,\rho]}^{r})$ will be shown.  Furthermore, by continuity and the increasing nature of $\mathbf{P}(\mathcal{A}_{[n,\rho]}^{r})$ in $r$, there exists $r_{0} = r_{0}(n,\rho,\epsilon)$ such that $\mathbf{P}(\mathcal{A}_{[n,\rho]}^{r_{0}}) = \frac{1}{2}$.  This particular radius of connectivity demarcates the transition from a set of disjoint collection of objects to a singly-colored, background object, almost surely by results in $\cite{Grimmett,Grimmett2,Meester}$.

In the special case of image segmentation, the unique $r_{0}$ is the maximum radius for use when defining clusters of correlated pixel intensities into disjoint regions.  For $\epsilon \in (0,\frac{1}{2})$, this radius of connectivity in the general case is the center of the estimated interval of radii, upon which, $\mathbf{P}(\mathcal{A}_{[n,\rho]}^{r})$ increases sharply from $\epsilon$ to $1-\epsilon$.  However, no partition of the bounded region exists in the initial formulation, as open edges exist between points generated by a point process if they are jointly within distance $r$.  In this case, the points are $r$-connected.  Then, for any $r > 0$, the bounded region can be partitioned into structures that depend on the choice of $r > 0$, suggesting another formulation where all continuity results in the continuum formulation still hold with a different critical radius, $r_{0}^{*}$.

It is shown that $r_{0} \le r_{0}^{*}$ in general and a condition is given which requires $r_{0} = r_{0}^{*}$.  In fact, it is shown that $r_{0} = r_{0}^{*}$ whenever the bounded region is partitioned into structures of size defined by $r \in (0,r_{0}]$, which is precisely the requirement for the special case of image segmentation.

Theory is concluded with a practical estimate for $r_{0}$ and its sharp threshold interval length.  The theory is demonstrated using a practical example from imaging.

\section{Random Geometric Graphs}
\label{rgg}

\subsection{Definitions}

\begin{definition}
A \textit{point process} is a mapping $X : \mathbb{R}^{2} \rightarrow \mathbb{R}^{2}$ such that for any subset $\mathcal{B} \subset \mathbb{R}^{2}$, there is an $n \in \mathbb{N}$ and a subset $\mathcal{X}_{n} = \{x_{k}\}_{1 \leq k \leq n} \subset \mathcal{B}$ with $X(\mathcal{B}) = \mathcal{X}_{n}$.
\end{definition}

\begin{definition}
Suppose $\mathcal{B} \subset \mathbb{R}^{2}$ and $X$ is a point process that randomly generates independent points $\mathcal{X}_{n} = \{x_{k}\}_{1 \leq k \leq n} \subset \mathcal{B}$ according to some probability distribution.  Let $d: \mathbb{R}^{2} \rightarrow \mathbb{R}$ be a distance measure.  Points $x,y \in \mathcal{X}_{n}$ are said to be \textit{$r$-connected} and form an \textit{$r$-open edge}, if $d(x,y) \leq r$, for some fixed $r > 0$.  Points $x,y \in \mathcal{X}_{n}$ are \textit{$r$-disconnected} and form an \textit{$r$-closed edge} otherwise.
\end{definition}

\begin{definition}
Let $E$ be the set of edges between points in $\mathcal{X}_{n}$.  $G(\mathcal{X}_{n};r)$ is defined to be the \textit{$r$-graph} of the set of all $r$-open and $r$-closed edges from $E$ between points in $\mathcal{X}_{n}$.
\end{definition}

\begin{definition}
Given points $x,y \in \mathcal{X}_{n}$ and some fixed $r > 0$, denote the $r$-edge between $x$ and $y$ as $<x,y>_{r}$.  A subset of points $C \subseteq \mathcal{X}_{n}$ forms an \textit{$r$-connected cluster} if and only if given any $x,y \in C$, there exists $r$-open edges $<x,a_{1}>_{r},<a_{1},a_{2}>_{r},...,<a_{k-1},y>_{r} \ \in E$ connecting $x$ to $y$, for points $\{a_{1},a_{2},...,a_{k-1}\} \subseteq C$.
\end{definition}

\begin{definition}
Let $\mathcal{A}$ be a set of graphs of $E$ and $G(\mathcal{X}_{n};r) \in \mathcal{A}$.  $\mathcal{A}$ is said to be an \textit{increasing property} if and only if for $r^{\prime} \neq r$ such that $G(\mathcal{X}_{n};r) \subseteq G(\mathcal{X}_{n};r^{\prime})$, we have $G(\mathcal{X}_{n};r^{\prime}) \in \mathcal{A}$.
\end{definition}

\begin{definition}
\label{sharp_threshold_def}
Let $\Omega$ be the set of values taken by the random variables $\mathcal{X}_{n}$, with $\mathcal{F}$ being any $\sigma$-algebra of subsets of $\Omega$ and $\mathbf{P}$ a probability measure on $(\Omega,\mathcal{F})$.  If $\mathcal{A}$ is a monotone (increasing) property and $\epsilon \in (0,\frac{1}{2})$, define
\begin{equation*}
r(n,\epsilon) = \inf \{r > 0 : \mathbf{P}(G(\mathcal{X}_{n};r) \in \mathcal{A}) \geq \epsilon\}
\end{equation*}
and
\begin{equation*}
\Delta(n,\epsilon) = r(n,1-\epsilon) - r(n,\epsilon).
\end{equation*}
If $\Delta(n,\epsilon) = o(1)$, then $\mathcal{A}$ has a \textit{sharp threshold}.
\end{definition}

\subsection{An Important Result}

\begin{theorem}
\textit{$\cite[Thm.\ (1.1)]{Goel}$}
\label{geometric_graph_sharp_threshold_length}
For increasing properties $\mathcal{A}$ consisting of graphs of points $\mathcal{X}_{n} \subset \mathbb{R}^{2}$,
\begin{eqnarray}
\label{geometric_graph_sharp_threshold_length_1}
\Delta(n,\epsilon) &=& \Theta(r_{c}\log^{1/4}(n)) \nonumber \\ &=& \Theta\left(\sqrt{\frac{1}{n\log{n}}}\right)
\end{eqnarray}
where
\begin{eqnarray}
\label{geometric_graph_sharp_threshold_length_2}
r_{c} = O\left(\sqrt{\frac{\log{n}}{n}}\right).
\end{eqnarray}
\end{theorem}

Without loss of generality, we may assume that the max of all radii is bounded such that $r_{max} \le 1$, since $\mathbb{R}^{2}$ can be continuously and bijectively mapped to a unit square in the $2$-dimensional plane.  Now, in the previous section, a point process generates points in some higher dimensional space, which are then projected to $\mathbb{R}^{2}$ and uniformly distributed throughout the unit square.  Since $r_{c} \le r_{max} \le 1$, then thm. $(\ref{geometric_graph_sharp_threshold_length})$ can be interpreted to mean, long range $r_{c}$-connectivity exists throughout $\mathcal{B} \subset \mathbb{R}^{2}$ if and only if every set of $\log{n}$ out of $n$ points are $r_{c}$-connected [eq. $(\ref{geometric_graph_sharp_threshold_length_2})$].  Then, a single $r_{c}$-connected cluster almost surely forms and the distance between points in every $\log{n}$ cluster is small [eq. $(\ref{geometric_graph_sharp_threshold_length_1})$].

Lastly, let $\mathcal{B} \subset \mathbb{R}^{2}$ be the unit square centered at $\hat{0} = (0,0)$.  For $r > r_{c}$, the length of the sharp threshold interval of radii, upon which the probability of long range $r$-connectivity sharply increases to $1$, is given by eq. $(\ref{geometric_graph_sharp_threshold_length_1})$.  The estimated length of the interval cannot be improved upon, given this reasoning.

\subsection{Definitions}

\begin{definition}
Given a fixed point, $y \in \mathcal{X}_{n}$, an \textit{$r$-connected component} containing $y$ is the subset of points $<C_{y}>_{r} \ \subseteq \mathcal{X}_{n}$ containing $y$ and every $x \in \mathcal{X}_{n} \backslash \{y\}$ having a set of $r$-open edges connecting $x$ to $y$.
\end{definition}

\begin{definition}
Given an $r$-open edge, $e = \ <x,y>_{r} \ \in G(\mathcal{X}_{n};r)$, an \textit{$r$-connected component} containing $e$ is the subset of points $<C_{e}>_{r} \ \subseteq \mathcal{X}_{n}$ containing $x$ and $y$ together with every $z \in \mathcal{X}_{n} \backslash \{x,y\}$ having a set of $r$-open edges connecting $z$ to both $x$ and $y$.
\end{definition}

\begin{definition}
Let $\mathcal{E}$ be any $\sigma$-algebra of subsets of $E$ such that $\emptyset, E \in \mathcal{E}$, any $A \in \mathcal{E}$ implies $A^{c} \in \mathcal{E}$ and all countable unions of subsets of $\mathcal{E}$ is again in $\mathcal{E}$.  Suppose $\{\eta_{k}\}_{k \geq 1}$ is a sequence of random variables on $E$ taking values in $\mathbb{R}$.  It will be said that $\eta_{k}$ \textit{converges weakly} to a random variable $\eta : E \rightarrow \mathbb{R}$ $($written $\eta_{k} \Rightarrow \eta)$, if
\begin{eqnarray*}
\lim_{k \rightarrow \infty} F_{k}(x) & = & \lim_{k \rightarrow \infty} \mathbf{P}(\eta_{k} \leq x) \\ & = & \mathbf{P}(\eta \leq x) \\ & = & F(x)
\end{eqnarray*}
for all $x \in \mathbb{R}$.
\end{definition}

\subsection{The Event}

\subsubsection{Bounded Number of Points}
\label{continuum_bounded_event}

Let $<C>_{r} \ \subseteq \mathcal{X}_{n}$ be an $r$-connected component containing $\hat{0}$ such that $|<C>_{r}| = \mathcal{N}$ and define $\rho_{n}(C) = \frac{\mathcal{N}}{n}$.  For $\rho \in (\frac{1}{2},1)$, define the graph property of all connected components containing at least half of all available points by
\begin{equation}
\label{continuum_event_finite}
\mathcal{A}_{[n,\rho]}^{r} = \{<C>_{r} \ \subseteq \mathcal{X}_{n} : \mathbf{E}[\ \rho_{n}(C)\ ] \geq \rho\}.
\end{equation}
As in $\cite{Goel}$, for $\epsilon \in (0,\frac{1}{2})$, define
\begin{equation}
\label{r_n_rho_epsilon_eq}
r(n,\rho,\epsilon) = \inf\{r > 0 : \mathbf{P}(\mathcal{A}_{[n,\rho]}^{r}) \geq \epsilon\}
\end{equation}
to be the critical radius at which $\mathcal{A}_{[n,\rho]}^{r}$ occurs with probability at least $\epsilon$ and define
\begin{equation}
\Delta(n,\rho,\epsilon) = r(n,\rho,1-\epsilon) - r(n,\rho,\epsilon)
\end{equation}
to be the length of the continuum of radii upon which $\mathcal{A}_{[n,\rho]}^{r}$ increases in probability of occurrence from $\epsilon > 0$ to $1 - \epsilon > 0$.

\subsubsection{Unbounded Number of Points}

In the case of $n$ being unbounded, define the corresponding graph property to be
\begin{equation}
\mathcal{A}^{r} = \{<C>_{r} \ \subseteq \mathcal{X}_{\infty} : |<C>_{r}| = \infty\}.
\end{equation}
Define
\begin{equation}
r(\epsilon) = \inf\{r > 0 : \mathbf{P}(\mathcal{A}^{r}) \geq \epsilon\}
\end{equation}
to be the critical radius at which $\mathcal{A}^{r}$ occurs with probability at least $\epsilon$ and define
\begin{equation}
\Delta(\epsilon) = r(1-\epsilon) - r(\epsilon)
\end{equation}
to be the length of the continuum of radii upon which $\mathcal{A}^{r}$ increases in probability of occurrence from $\epsilon > 0$ to $1 - \epsilon > 0$.

\subsection{Continuity Results}
\label{continuum_continuity_results}

Recall from sec. $(\ref{proc})$, one of the first steps in segmenting an image in the continuum is to view an image as a set of observations from a point process that samples intensities from a deterministic $2$-dimensional surface that has been distorted with errors from loss of higher dimensional information, along with added blurring and other degradations supplied by a zero-mean, finite-variance, normal distribution.  Correlations between sampled intensities are determined by spatial distance and differences in intensity values to reveal a graph structure of open/closed edges as a function of a connectivity radius, $r > 0$.  Defining an increasing graph property as the set of all $r$-connected points, we want to know the existence and uniqueness of $r_{0} > 0$ such that this graph property contains at least half of all points generated by the point process, a critical measure.

Since the samples are random, then we will formulate continuity conditions for a probability measure defined over a $\sigma$-algebra containing the graph property.  The property is increasing as a function of the connectivity radius.  Thus, we can make the determination of the existence of a critical radius, beyond which, we are virtually assured (by results in $\cite{Meester}$) of half of all generated points in the continuum being contained in the property.

\begin{theorem}
\textit{$\cite[Thm.\ (3.8)]{Meester}$}
\label{meester_theorem_3_8}
Suppose $\{r_{k}\}_{k \geq 1}$ is a sequence of radii such that $0 < r_{k} \leq R$ for some $R > 0$ and $\{\eta_{k}\}_{k \geq 1}$ is a sequence of random variables which take values $r_{k}$ with probability $1$.  If $0 < r \leq R$ and $\eta$ is a random variable taking the value $r$ with probability $1$ such that $\eta_{k} \Rightarrow \eta$ as $k \rightarrow \infty$.  Then, $\mathbf{P}(\mathcal{A}^{\eta_{k}}) \rightarrow \mathbf{P}(\mathcal{A}^{\eta})$ as $k \rightarrow \infty$.
\end{theorem}

\begin{proof}
This is just a restatement of $\cite[Thm.\ (3.8)]{Meester}$ for the special case of random variables $\eta_{k}$ and $\eta$ such that $\mathbf{P}(\eta_{k} = r_{k}) = 1 = \mathbf{P}(\eta = r)$ for all $k \geq 1$.  
\end{proof}

\begin{corollary}
\textit{(to Theorem $\ref{meester_theorem_3_8}$)}
\label{continuum_continuity_corollary}
$\mathbf{P}(\mathcal{A}_{[n,\rho]}^{r})$ is a continuous function of $r$.
\end{corollary}

\begin{proof}
Continuity of $\mathbf{P}(\mathcal{A}^{r})$ in $r$ follows from thm. $(\ref{meester_theorem_3_8})$.  Therefore, the result follows by noting that a bijection of $\mathbb{R}^{2}$ into the unit square centered at the origin requires $\mathbf{P}(\mathcal{A}^{r})$ to be continuous in $r$ on any bounded region, $\mathcal{B} \subset \mathbb{R}^{2}$ such that $\mathcal{X}_{n} \subset \mathcal{B}$.  
\end{proof}

\begin{theorem}
\label{continuum_iff}
$r = r(n,\rho,\epsilon)$ is a continuous function of $\epsilon$ if and only if $\mathbf{P}(\mathcal{A}_{[n,\rho]}^{r})$ is a continuous function of $r$.
\end{theorem}

\begin{proof}
Suppose $r(n,\rho,\epsilon)$ is a continuous function of $\epsilon$ and let $\{\epsilon_{k} \in (0,\frac{1}{2})\}_{k \geq 1}$ be a sequence of positive real numbers such that $\epsilon_{k} \rightarrow \epsilon_{0}$ as $k \rightarrow \infty$.  Let $\{X(e)\}_{e \in G(\mathcal{X}_{n};r)}$ be a finite sequence of uniformly distributed random variables with values in $[0,1]$ and define a sequence of random variables $\{\eta_{k}\}_{k \geq 1}$ by $\eta_{k}(e) = r(n,\rho,\epsilon_{k}) \equiv r_{k}$ when $X(e) < 1$ and $0$ otherwise.  Clearly, $\eta_{k} = r_{k}$ with probability $1$ for all $k \geq 1$.  Likewise, define a random variable $\eta_{0}$ by $\eta_{0}(e) = r(n,\rho,\epsilon_{0}) \equiv r_{0}$ when $X(e) < 1$ and $0$ otherwise so that $\eta_{0} = r_{0}$ with probability $1$.  Since $r(n,\rho,\epsilon)$ is continuous in $\epsilon$, then $r_{k} \rightarrow r_{0}$ as $k \rightarrow \infty$ so that $\eta_{k} \Rightarrow \eta_{0}$ as $k \rightarrow \infty$.  Now, define $R = 2 * \max \{d(x,y) : x,y \in \mathcal{X}_{n}\}$.  By lem. $(\ref{continuum_connection_radius_bound_R})$, $0 < \eta_{k} \leq R$ for all $k \geq 0$.  Therefore, $\mathbf{P}(\mathcal{A}_{[n,\rho]}^{\eta_{k}}) \rightarrow \mathbf{P}(\mathcal{A}_{[n,\rho]}^{\eta_{0}})$ as $k \rightarrow \infty$ by cor. $(\ref{continuum_continuity_corollary})$ since $r_{k} \rightarrow r_{0}$ as $k \rightarrow \infty$.  Thus, $\mathbf{P}(\mathcal{A}_{[n,\rho]}^{r})$ is a continuous function of $r$.  Conversely, suppose $\mathbf{P}(\mathcal{A}_{[n,\rho]}^{r})$ is a continuous function of $r$ and let $\{\epsilon_{k} \in (0,\frac{1}{2})\}_{k \geq 1}$ be any convergent sequence such that $\epsilon_{k} \rightarrow \epsilon_{0}$.   Define $r_{k} = r(n,\rho,\epsilon_{k})$ and $r_{0} = r(n,\rho,\epsilon_{0})$.  Given $\xi > 0$, then $\Xi \equiv \{k \geq 1 : |\mathbf{P}(\mathcal{A}_{[n,\rho]}^{r_{k}}) - \mathbf{P}(\mathcal{A}_{[n,\rho]}^{r_{0}})| \geq \xi\}$ is a set of measure zero by the continuity assumption.  Therefore, $r_{k} \rightarrow r_{0}$ as $k \rightarrow \infty$ by prop. $(\ref{continuum_connection_radius_convergent_sequence})$.  Thus, suppose that $\Xi \neq \emptyset$.  Then, $\Xi$ is at most countable so that $r_{k} \rightarrow r_{0}$, almost surely, as $k \rightarrow \infty$, by prop. $(\ref{continuum_connection_radius_convergent_sequence})$.  Thus, $r(n,\rho,\epsilon)$ is a continuous function of $\epsilon$.  
\end{proof}

\begin{lemma}
\label{existence_r_0_lemma}
There exists $r_{0} > 0$ such that $\mathbf{P}(\mathcal{A}_{[n,\rho]}^{r_{0}}) = \frac{1}{2}$.
\end{lemma}

\begin{proof}
By properties of probabilities measures, $\mathbf{P}(\mathcal{A}_{[n,\rho]}^{r}) \in [0,1]$ and by prop. $(\ref{continuum_probability_measure_non_decreasing_r})$, $\mathbf{P}(\mathcal{A}_{[n,\rho]}^{r})$ is non-decreasing as a function of increasing $r > 0$.  By thm. $(\ref{geometric_graph_sharp_threshold_length})$, $\mathbf{P}(\mathcal{A}_{[n,\rho]}^{r})$ increases from $\epsilon > 0$ to $1 - \epsilon > 0$ for fixed $\epsilon \in (0,\frac{1}{2})$.  The result follows by cor. $(\ref{continuum_continuity_corollary})$.  
\end{proof}

\begin{lemma}
\label{existence_r_0_interior_lemma}
For any compact continuum of radii $I$, define $\mathbf{P}(\mathcal{A}_{[n,\rho]}^{I}) := \{\mathbf{P}(\mathcal{A}_{[n,\rho]}^{r}) : r \in I\}$.  Then, $r_{0}$ is in the interior of any compact interval $I_{\epsilon}$ such that $\mathbf{P}(\mathcal{A}_{[n,\rho]}^{I_{\epsilon}}) = [\epsilon,1-\epsilon]$.
\end{lemma}

\begin{proof}
If $I$ is compact, then by cor. $(\ref{continuum_continuity_corollary})$, there exists $x,y \in \mathbb{R}$ such that $\mathbf{P}(\mathcal{A}_{[n,\rho]}^{I}) = [x,y]$.  Without loss of generality, we may assusme that $I = I_{\epsilon}$ for $\epsilon \in (0,\frac{1}{2})$ and $[x,y] = [\epsilon,1-\epsilon]$ since we can take inspiration from a bijection of $\mathbb{R}^{2}$ into the unit square centered at the origin to define the bijection
\begin{equation}
\label{existence_r_0_interior_lemma_eq}
r \mapsto
\begin{cases}
\frac{1-2\epsilon}{y-x}\big(r-x\big)+\epsilon,\ r \in [x,y]\ \&\ x < y \\
\frac{1}{2},\ \ \ \ \ \ \ \ \ \ \ \ \ \ \ \ \ \ \ x = r = y
\end{cases}
\end{equation}
for fixed $\epsilon \in (0,\frac{1}{2})$.  Now, seeking a contradiction, suppose $r_{0}$ is in the boundary of $I_{\epsilon}$.  Then, $\mathbf{P}(\mathcal{A}_{[n,\rho]}^{r_{0}}) = \epsilon$ or $\mathbf{P}(\mathcal{A}_{[n,\rho]}^{r_{0}}) = 1-\epsilon$ by compactness of $I_{\epsilon}$ and continuity of $\mathbf{P}(\mathcal{A}_{[n,\rho]}^{r})$ in $r$ by cor. $(\ref{continuum_continuity_corollary})$.  By lem. $(\ref{existence_r_0_lemma})$, $\epsilon = \frac{1}{2}$ in either case.  This is a contradiction since $\epsilon \in (0,\frac{1}{2})$.  
\end{proof}

\begin{remark}
Note that if $|I_{\epsilon}| = 1$ (when $x = y$) then $I_{\epsilon}=\{r_{0}\}$ and $\mathbf{P}(\mathcal{A}_{[n,\rho]}^{r_{0}}) = \frac{1}{2}$ by lem. $(\ref{existence_r_0_lemma})$ so that the mapping defined in eq. $(\ref{existence_r_0_interior_lemma_eq})$ is consistent.
\end{remark}

\begin{lemma}
\label{existence_r_0_interior_all_lemma}
If $r_{0}$ is independent of $\epsilon$, then $r_{0} \in I_{\epsilon}$ for all $\epsilon \in (0,\frac{1}{2})$.
\end{lemma}

\begin{proof}
Note that $r_{0} \in I = \bigcap_{k} I_{\epsilon_{k}}$ for any sequence $\epsilon_{k} \rightarrow \frac{1}{2}$.  Clearly $I$ is compact so that $r_{0}$ is in the interior of $I$.  Therefore, either $I$ is an interval or $I = \{r_{0}\}$.  Suppose $I$ is an interval of radii.  Since $r_{0}$ is in the interior of $I$, then there exists $r_{0}^{\prime} < r_{0} \in I$.  Now, since $\epsilon_{k} \rightarrow \frac{1}{2}$, then $\mathbf{P}(\mathcal{A}_{[n,\rho]}^{r_{0}^{\prime}}) = \frac{1}{2}$ and $r_{0}^{\prime} < r_{0} = \inf \{r > 0 : \mathbf{P}(\mathcal{A}_{[n,\rho]}^{r}) = \frac{1}{2}\}$.  This is a contradiction.  Therefore, $I = \{r_{0}\}$ so that $r_{0}$ is unique as a function of $\epsilon \in (0,\frac{1}{2})$.  
\end{proof}

Currently, we have established continuity of the probability measure defined over graph properties in the continuum.  Also established is the existence of a critical connectivity radius, $r_{0}$ as a function of the number of observations, $n$, proportion of connected observations, $\rho$, and the probability, $\epsilon$.  Establishing the independence of $r_{0}$ from $\epsilon$ gives uniqueness of $r_{0}$ by lem. $(\ref{existence_r_0_interior_all_lemma})$.

\subsection{Continuum Giant Component}
\label{continuum_giant_component}

\begin{theorem}
\label{continuum_existence_r_0}
There exists $r_{0} = r_{0}(n,\rho,\epsilon) < \infty$, independent of $\epsilon$, such that
\begin{eqnarray*}
\mathbf{P}(\mathcal{A}_{[n,\rho]}^{r_{0}}) = \frac{1}{2}.
\end{eqnarray*}
\end{theorem}

\begin{proof}
Let $\epsilon \in (0,\frac{1}{2})$ be given.  Since $\mathcal{A}_{[n,\rho]}^{r}$ is an increasing property in $r$ by prop. $(\ref{continuum_increasing_property_r})$, thm. $(\ref{geometric_graph_sharp_threshold_length})$ applies.  Thus, there exists an interval $I_{\epsilon}$ of length $\Delta(n,\rho,\epsilon)$ such that $\mathbf{P}(\mathcal{A}_{[n,\rho]}^{r}) \in [\epsilon,1-\epsilon]$ for $r \in I_{\epsilon}$.  Since $\mathbf{P}(\mathcal{A}_{[n,\rho]}^{r})$ is a continuous function of $r$ by cor. $(\ref{continuum_continuity_corollary})$ and non-decreasing in $r$ by prop. $(\ref{continuum_probability_measure_non_decreasing_r})$ and $\frac{1}{2} \in [\epsilon,1-\epsilon]$, then there exists $r_{0} \in I_{\epsilon}$ such that $\mathbf{P}(\mathcal{A}_{[n,\rho]}^{r_{0}}) = \frac{1}{2}$.  If $R = 2 * \max \{d(x,y) : x,y \in \mathcal{X}_{n}\}$, then by lem. $(\ref{continuum_connection_radius_bound_R})$, we have $0 < r_{0}(n,\rho,\epsilon) \leq R < \infty$.  It remains to be shown that $r_{0} = r_{0}(n,\rho,\epsilon)$, independent of $\epsilon$.  
\end{proof}

Recall that $\rho \in (\frac{1}{2},1)$ and note that the maximum distance between any two connected points in $\mathcal{B}$ is inversely proportional to $n$ by eq. $(\ref{geometric_graph_sharp_threshold_length_1})$.  Then, the particular $r_{0}$, which meets the requirements of thm. $(\ref{continuum_existence_r_0})$, is the exact radius, such that, it is equally probable (than not) that more than half of all points are connected contiguously.  Only one such cluster exists, with all other clusters being disjoint and sparsely connected throughout $\mathcal{B}$.  As such, $r_{0}$ demarcates the radial connection length at which the property $\mathcal{A}_{[n,\rho]}^{r}$ undergoes a phase transition such that the graph $G(\mathcal{X}_{n};r)$ is likely to be sparsely connected $\cite[Thms.\ (3.3,3.6)]{Meester}$, almost surely, when $r \in (0,r_{0}]$.  Conversely, $G(\mathcal{X}_{n};r)$ is more likely to be fully connected and form one connected cluster of points $\cite[Thms.\ (3.3,3.6)]{Meester}$, almost surely, when $r \in (r_{0},1]$.  Likewise, in terms of image segmentation, $r_{0}$ denotes the transition from something "interesting", with multiple objects, to something that almost surely consists of only one color.

\begin{lemma}
\label{continuum_independence_r_0_epsilon}
$r_{0} = r_{0}(n,\rho,\epsilon)$ is independent of $\epsilon$.
\end{lemma}

\begin{proof}
Let $\epsilon_{1},\epsilon_{2} \in (0,\frac{1}{2})$ and suppose $r_{0,1} = r_{0}(n,\rho,\epsilon_{1}),r_{0,2} = r_{0}(n,\rho,\epsilon_{2})$ such that
\begin{equation}
\label{r1_r2_eq}
\mathbf{P}(\mathcal{A}_{[n,\rho]}^{r_{0,1}}) = \frac{1}{2} = \mathbf{P}(\mathcal{A}_{[n,\rho]}^{r_{0,2}}).
\end{equation}
It has to be shown that $r_{0,1} = r_{0,2}$.  Let $\{\epsilon_{k}\}_{k \geq 1}$ be a sequence such that $\epsilon_{k} = \epsilon_{1}$ for all $k \geq 1$ and define $r_{0,k} = r_{0}(n,\rho,\epsilon_{k})$.  For fixed $n,\rho$ and arbitrary $\xi > 0$, define $\Xi_{i}^{\xi} = \{k \geq 1 : |\mathbf{P}(\mathcal{A}_{[n,\rho]}^{r_{0,k}})-\mathbf{P}(\mathcal{A}_{[n,\rho]}^{r_{0,i}})| \geq \xi\}$ for $i \in {1,2}$.  Then, by eq. $(\ref{r1_r2_eq})$, we have $\emptyset=\Xi_{1}^{\xi}=\Xi_{2}^{\xi}$, since $r_{0,k} = r_{0,1}$ for all $k \geq 1$.  Hence, by prop. $(\ref{continuum_connection_radius_convergent_sequence})$, $r_{0,k} \rightarrow r_{0,2}$ as $k \rightarrow \infty$.  But, $r_{0,k} = r_{0,1}$ for all $k \geq 1$ so that $r_{0,1} = r_{0,2}$.  Thus, $r_{0} = r_{0}(n,\rho)$, independent of $\epsilon$.  
\end{proof}

\begin{remark}
As a result of thm. $(\ref{continuum_independence_r_0_epsilon})$, $r(\epsilon)$ is independent of $\epsilon > 0$, since $r(n,\rho,\epsilon) \rightarrow r(\epsilon)$ as $\mathbf{E}[n] \rightarrow \infty$.  As such, $\Delta(\epsilon) = o(1)$ so that $\mathcal{A}^{r}$ has a sharp threshold, by definition $(\ref{sharp_threshold_def})$.
\end{remark}

\begin{corollary}
\label{continuum_critical_radius_unique}
The critical radius, associated with the property $\mathcal{A}^{r}$, is unique.
\end{corollary}

\begin{proof}
$r(\epsilon)$ is the limit of $r(n,\rho,\epsilon)$ as $\mathbf{E}[n] \rightarrow \infty$.  As such, $r_{0}$ is the constant limit of $r_{0}(n,\rho)$ as $\mathbf{E}[n] \rightarrow \infty$.  
\end{proof}

\begin{corollary}
\label{continuum_independence_r_0_epsilon_corollary}
Given $r > 0$, there exists a density of points $\lambda_{0} = \lambda(n_{0})$ such that
\begin{eqnarray*}
\mathbf{P}(\mathcal{A}_{[n_{0},\rho]}^{r}) = \frac{1}{2}.
\end{eqnarray*}
\end{corollary}

\begin{proof}
By lem. $(\ref{continuum_independence_r_0_epsilon})$, let $n_{0} = n_{0}(r,\rho)$ be the minimum of all positive (real) solutions to $r = r_{0}(n,\rho)$ for some fixed $r > 0$.  The result follows.  
\end{proof}

Since $n$ is inversely proportional to connection distance $r$ (requiring that $n \in [1,\infty)$), then the particular $n_{0}$, which meets the requirements of cor. $(\ref{continuum_independence_r_0_epsilon_corollary})$, is the exact number of points, such that, it is equally probable (than not) that more than half of all points are connected contiguously.  In this case, only one such cluster exists, with all other clusters being disjoint and sparsely connected throughout $\mathcal{B}$.  Otherwise, all connected clusters disjointly contain half (or less than half) of all available points, in which case, more than one such cluster can exist.  As such, $n_{0}$ demarcates the number of points at which the property $\mathcal{A}_{[n,\rho]}^{r}$ undergoes a phase transition such that the graph $G(\mathcal{X}_{n};r)$ is likely to be sparsely connected to form disjoint, connected clusters of points $\cite[Thms.\ (3.3,3.6)]{Meester}$, almost surely, when $n \in [1,n_{0}]$.  Conversely, $G(\mathcal{X}_{n};r)$ is more likely to be fully connected and form one connected cluster of points $\cite[Thms.\ (3.3,3.6)]{Meester}$, almost surely, when $n \in (n_{0},\infty)$.  Then, for image segmentation, $n_{0}$ denotes the transition from being sparsely populated, with multiple distinct objects, to one that almost surely consists of only one object of note, a densely-populated, singly-colored background.

\subsection{Continuum Sharp Threshold Interval Length}

Given the particular radius guaranteed by thm. $(\ref{continuum_existence_r_0})$, then thm. $(\ref{geometric_graph_sharp_threshold_length})$ can be used to find an estimate of the length of the sharp threshold interval such that $\mathbf{P}(\mathcal{A}_{[n,\rho]}^{r})$ increases sharply from some $\epsilon \in (0,\frac{1}{2})$ to $1 - \epsilon$.  By lem. $(\ref{continuum_independence_r_0_epsilon})$, $r_{0}$ is independent of any particular $\epsilon$.  Thus, the interval and its length must be fixed, given $n$ and $\rho \in (\frac{1}{2},1)$.

\begin{theorem}
\label{continuum_goel_sharp_threshold_length}
$\Delta(n,\rho) = \Theta(r_{0}\log^{\frac{1}{4}}n)$.
\end{theorem}

\begin{proof}
For $\delta \in (0,\frac{1}{2})$, let $\epsilon_{\delta} = \frac{1}{2} - \delta$.  By thm. $(\ref{geometric_graph_sharp_threshold_length})$ and thms. $(\ref{continuum_existence_r_0})$ and $(\ref{continuum_independence_r_0_epsilon})$,
\begin{eqnarray*}
\Delta(n,\rho) & = & \lim_{\delta \rightarrow 0^{+}} \Delta(n,\rho,\epsilon_{\delta}) \\ & = & \lim_{\delta \rightarrow 0^{+}} \Theta(r(n,\rho,\epsilon_{\delta})\log^{\frac{1}{4}}n) \\ & = & \Theta(r_{0}\log^{\frac{1}{4}}n).
\end{eqnarray*}  
\end{proof}

Theorem $(\ref{continuum_goel_sharp_threshold_length})$ gives an expected result, given thm. $(\ref{geometric_graph_sharp_threshold_length})$ above.  However, in $\cite{Cai}$, a much more practical estimate of this length is obtained after the bounded region is partitioned by hexagons of a known size.  If $M^{2}$ is the number of these hexagons in the bounded region, then it is shown that a good estimate of the sharp threshold interval length is a polynomial in $1/M$.

\begin{theorem}
\label{continuum_probability_upper_lower_bound}
There is a constant $c > 0$, independent of $M$, such that for all $\epsilon_{1} > 0$ and every fixed small $\delta > 0$
\begin{eqnarray}
\label{continuum_probability_upper_lower_bound_1}
\mathbf{P}(\mathcal{A}_{[n,\rho+\delta]}^{r}) \leq (\frac{1}{2} + \epsilon_{1})M^{-c(r_{0} - r)}
\end{eqnarray}
for all $r \leq r_{0}$ and
\begin{eqnarray}
\label{continuum_probability_upper_lower_bound_2}
\mathbf{P}(\mathcal{A}_{[n,\rho-\delta]}^{r}) \geq 1 - (\frac{1}{2} + \epsilon_{1})M^{-c(r - r_{0})}
\end{eqnarray}
for all $r \geq r_{0}$.
\end{theorem}

\begin{theorem}
\label{continuum_continuity_rho}
$\mathbf{P}(\mathcal{A}_{[n,\rho]}^{r})$ is a continuous function of $\rho$.
\end{theorem}

\section{Hexagonal Partition Model}
\label{hpm}

It was seen in section $(\ref{continuum_giant_component})$ that $r_{0} > 0$ exists such that the probability is $\frac{1}{2}$ for the occurrence of the property that at least half of all points connect in the bounded region, $\mathcal{B}$.  By thm. $(\ref{geometric_graph_sharp_threshold_length})$,
\begin{equation}
r_{c} = r_{c}(n) = O\left(\sqrt{\frac{\log{n}}{n}}\right) < r_{0}(n) = r_{0}
\end{equation}
where $r_{c}$ defines the critical radius at which the continuum property occurs with arbitrarily small, positive probability.

For fixed $r \in (r_{c},r_{0}]$, let $h^{r}$ be the largest hexagon that can be inscribed into a circle of radius $r/4$.  Let $H_{r}$ be a countably infinite collection of copies of $h^{r}$ such that
\begin{equation}
\mathbb{R}^{2} = \bigcup_{h_{i,j}^{r} \in H_{r}} h_{i,j}^{r}
\end{equation}
and for $h_{i,j}^{r},h_{i^{\prime},j^{\prime}}^{r} \in H_{r}$, we have $h_{i,j}^{r} \neq h_{i^{\prime},j^{\prime}}^{r}$ whenever $|i - i^{\prime}| + |j - j^{\prime}| \neq 0$.  Connectivity between $x,y \in \mathcal{X}_{n}$ is then defined as $x$ and $y$ both lying in the same hexagon or neighboring hexagons.

With the bounded region $\mathcal{B}$ partitioned into hexagons, the analysis proceeds, whereby the original problem of estimating the sharp threshold interval length in the continuum is now replaced by a similar problem in the hexagonal partition.  Connectivity and the increasing property are redefined.  Continuity and existence results are shown to still hold.  Later, an analogue to thm. $(\ref{continuum_probability_upper_lower_bound})$ is stated and proven.

\subsection{Definitions}

\begin{definition}
A \textit{hexagonal partition} of $\mathcal{B}$ is a finite collection of hexagons from $H_{r}$ such that $\mathcal{B}$ is a union of all hexagons in the finite collection.
\end{definition}

\begin{definition}
The \textit{Hamming distance} between elements, $h_{i,j}^{r},h_{i^{\prime},j^{\prime}}^{r} \in H_{r}$ is defined to be the quantity \begin{eqnarray*}
h(h_{i,j}^{r},h_{i^{\prime},j^{\prime}}^{r}) = |i - i^{\prime}| + |j - j^{\prime}|.
\end{eqnarray*}
\end{definition}

\begin{definition}
Points $x,y \in \mathcal{X}_{n}$ are \textit{$H_{r}$-connected} and $<x,y>_{H_{r}}$ is an \textit{$H_{r}$-open edge}, if there exists $h_{i_{x},j_{x}}^{r},h_{i_{y},j_{y}}^{r} \in H_{r}$ such that $x \in h_{i_{x},j_{x}}^{r}$ and $y \in h_{i_{y},j_{y}}^{r}$ where $h(h_{i_{x},j_{x}}^{r},h_{i_{y},j_{y}}^{r}) \leq 2$ with $|i_{x} - i_{y}| \leq 1$ and $|j_{x} - j_{y}| \leq 1$.  Points in $\mathcal{X}_{n}$ are \textit{$H_{r}$-disconnected} and form an \textit{$H_{r}$-closed edge} otherwise.
\end{definition}

\begin{definition}
Given a $y \in \mathcal{X}_{n}$, an \textit{$H_{r}$-connected component} containing $y$ is the subset of points $<C_{y}>_{H_{r}} \ \subseteq \mathcal{X}_{n}$ containing $y$ and every $x \in \mathcal{X}_{n} \backslash \{y\}$ having an $H_{r}$-open set of edges connecting $x$ to $y$.
\end{definition}

\begin{definition}
Given an $H_{r}$-connected edge, $e = <x,y>_{H_{r}}$, an \textit{$H_{r}$-connected component} containing $e$ is the subset of points $<C_{e}>_{H_{r}} \ \subseteq \mathcal{X}_{n}$ containing $x$ and $y$ and every $z \in \mathcal{X}_{n} \backslash \{x,y\}$ having an $H_{r}$-open set of edges connecting $z$ to both $x$ and $y$.
\end{definition}

\subsection{The Increasing Property}

\subsubsection{Bounded Number of Points}

Let $<C>_{H_{r}} \ \subseteq \mathcal{X}_{n}$ be an $H_{r}$-connected component such that $|<C>_{H_{r}}| = \mathcal{N}$ and let $\rho_{n}(C) = \frac{\mathcal{N}}{n}$ be defined as in section $(\ref{continuum_bounded_event})$.  For $\rho \in (\frac{1}{2},1)$, define the graph property of all connected components containing at least $100\rho\%$ of all available points by
\begin{equation}
\label{hexagonal_partition_event_finite}
\mathcal{A}_{[n,\rho]}^{H_{r}} = \{<C>_{H_{r}} \ \subseteq \mathcal{X}_{n} : \mathbf{E}[\ \rho_{n}(C)\ ] \geq \rho\}.
\end{equation}
As in $\cite{Goel}$, for $\epsilon \in (0,\frac{1}{2})$, define
\begin{equation}
r^{*}(n,\rho,\epsilon) = \inf\{r > 0 : \mathbf{P}(\mathcal{A}_{[n,\rho]}^{H_{r}}) \geq \epsilon\}
\end{equation}
to be the critical radius at which $\mathcal{A}_{[n,\rho]}^{H_{r}}$ occurs with probability at least $\epsilon$ and define
\begin{equation}
\Delta^{*}(n,\rho,\epsilon) = r^{*}(n,\rho,1-\epsilon) - r^{*}(n,\rho,\epsilon)
\end{equation}
to be the length of the continuum of radii upon which $\mathcal{A}_{[n,\rho]}^{H_{r}}$ increases in probability of occurrence from $\epsilon > 0$ to $1 - \epsilon > 0$.

\subsubsection{Unbounded Number of Points}

In the event that $n$ is unbounded, define the corresponding graph property to be
\begin{equation}
\mathcal{A}^{H_{r}} = \{<C>_{H_{r}} \ \subseteq \mathcal{X}_{\infty} : |<C>_{H_{r}}| = \infty\}.
\end{equation}
Define
\begin{equation}
r^{*}(\epsilon) = \inf\{r > 0 : \mathbf{P}(\mathcal{A}^{H_{r}}) \geq \epsilon\}
\end{equation}
to be the critical radius at which $\mathcal{A}^{H_{r}}$ occurs with probability at least $\epsilon$ and define
\begin{equation}
\Delta^{*}(\epsilon) = r^{*}(1-\epsilon) - r^{*}(\epsilon)
\end{equation}
to be the length of the continuum of radii upon which $\mathcal{A}^{H_{r}}$ increases in probability of occurrence from $\epsilon > 0$ to $1 - \epsilon > 0$.

\subsection{Continuity Results and Some Continuum Relationships}

The continuity results of section $(\ref{continuum_continuity_results})$ hold for the properties defined after the bounded region $\mathcal{B}$ is partitioned by copies of the hexagon $h^{r}$, since connectivity is now characterized by points lying within distance $r/2$ (within neighboring hexagons).  As such, hexagonal connectivity is another way of viewing connectivity in the continuum.  Then, by thm. $(\ref{continuum_existence_r_0})$, there exists $r_{0}^{*} = r_{0}^{*}(n,\rho)$ which satisfies the criteria of the theorem for the property $A_{[n,\rho]}^{H_{r}}$.

\begin{definition}
$G(\mathcal{X}_{n};H_{r})$ is defined to be the \textit{$H_{r}$-graph} of all $H_{r}$-open and $H_{r}$-closed edges between points in $\mathcal{X}_{n} \subset \mathcal{B}$.
\end{definition}

In addition to the continuity results under $r$-connectivity also holding under $H_{r}$-connectivity, the next lemma shows that the graph of the set of clusters formed under $H_{r}$-connectivity is a sub-graph of the set of clusters formed under $r$-connectivity.

\begin{lemma}
\label{graph_containment_lemma}
$G(\mathcal{X}_{n};H_{r}) \subseteq G(\mathcal{X}_{n};r)$.
\end{lemma}

\begin{proof}
Suppose $<x,y>_{H_{r}} \ \in G(\mathcal{X}_{n};H_{r})$ is any $H_{r}$-connected edge.  Without loss of generality, choose a coordinate system on $\mathbb{R}^{2}$ so that $<x,y>_{H_{r}}$ lies on a coordinate axis with $\hat{0} = (0,0)$ defined such that $d(x,\hat{0}) = \frac{d(x,y)}{2} = d(\hat{0},y)$.  Since $x,y \in \mathcal{X}_{n} \subset \mathcal{B}$ and $H_{r}$ is a partition of $\mathcal{B}$, then there exists $h_{i_{x},j_{x}}^{r},h_{i_{y},j_{y}}^{r} \in H_{r}$ such that $x \in h_{i_{x},j_{x}}^{r},y \in h_{i_{y},j_{y}}^{r}$ and $h(h_{i_{x},j_{x}}^{r},h_{i_{y},j_{y}}^{r}) \leq \max \{|i_{x}-i_{y}|,|j_{x}-j_{y}|\} \leq 1$.  Each of $h_{i_{x},j_{x}}^{r}$ and $h_{i_{y},j_{y}}^{r}$ are copies of $h^{r}$ and can be inscribed into copies of a circle of radius $\frac{r}{4}$.  Therefore, $d(x,y) = d(x,\partial{h_{i_{x},j_{x}}^{r}}) + d(\partial{h_{i_{y},j_{y}}^{r}},y) \leq \frac{r}{2} + \frac{r}{2} = r$ so that $x,y \in \mathcal{X}_{n}$ are $r$-connected.  Thus, $<x,y>_{H_{r}} \ \in G(\mathcal{X}_{n};r)$, which shows that $G(\mathcal{X}_{n};H_{r}) \subseteq G(\mathcal{X}_{n};r)$.  
\end{proof}

Using lem. $(\ref{graph_containment_lemma})$, the next results show that given a sample of size $n > 1$ and a connectivity radius $r > 0$, the probability of the event of one subset of connected data points containing at least half of the $n$ data points is (possibly) smaller under $H_{r}$-connectivity than under $r$-connectivity.  In addition, a (possibly) larger radius of connectivity is required to achieve the same proportion of data points being connected into one cluster.  Indeed, once an image is partitioned into pixels, segmentation into distinct objects is (possibly) more difficult, which may seem counter-intuitive.  For now, we accept this statement, as it is only a possibility.  Later, we will give a condition under which both formulations (continuum and hexagonal) are equivalent so that partitioning an image and segmenting it into distinct objects can be made, merging with our reality.  The resulting resolution is lower than that of the original image, as regions of higher resolution pixels are merged by full connectivity throughout disjoint regions of a uniform size determined by $r_{0}$, correlating with results in $\cite{Grimmett,Grimmett2}$.  These disjoint regions allow for the detection of transitions from open-to-closed edges, delineating the possibility of distinct objects.

\begin{lemma}
\label{probability_comparison_lemma}
$\mathbf{P}(\mathcal{A}_{[n,\rho]}^{H_{r}}) \leq \mathbf{P}(\mathcal{A}_{[n,\rho]}^{r})$.
\end{lemma}

\begin{proof}
By lem. $(\ref{graph_containment_lemma})$, $\mathcal{A}_{[n,\rho]}^{H_{r}} \subseteq \mathcal{A}_{[n,\rho]}^{r}$.  
\end{proof}

\begin{lemma}
\label{r_0_inequality}
$r_{0} \leq r_{0}^{*}$.
\end{lemma}

\begin{proof}
Seeking a contradiction, suppose $r_{0} > r_{0}^{*}$.  Then,
\begin{eqnarray}
\label{critical_radius_1} \frac{1}{2} & = & \mathbf{P}(\mathcal{A}_{[n,\rho]}^{r_{0}}) \\ \label{critical_radius_2} & \geq & \mathbf{P}(\mathcal{A}_{[n,\rho]}^{r_{0}^{*}}) \\ \label{critical_radius_4} & \geq & \mathbf{P}(\mathcal{A}_{[n,\rho]}^{H_{r_{0}^{*}}}) \\ \label{critical_radius_5} & = & \frac{1}{2}
\end{eqnarray}
where equality $(\ref{critical_radius_1})$ follows by thm. $(\ref{continuum_existence_r_0})$, ineq. $(\ref{critical_radius_2})$ follows by properties of probability measures and by hypothesis, ineq. $(\ref{critical_radius_4})$ follows by lem. $(\ref{probability_comparison_lemma})$ and equality $(\ref{critical_radius_5})$ follows by thm. $(\ref{continuum_existence_r_0})$.  It follows that $\mathbf{P}(\mathcal{A}_{[n,\rho]}^{r_{0}^{*}}) = \frac{1}{2}$.  Therefore, $r_{0}^{*} \in \{r > 0 : \mathbf{P}(\mathcal{A}_{[n,\rho]}^{r}) = \frac{1}{2}\}$ and $r_{0}^{*} < r_{0} = \inf \{r > 0 : \mathbf{P}(\mathcal{A}_{[n,\rho]}^{r}) = \frac{1}{2}\}$.  This is a contradiction.  Thus, $r_{0} \leq r_{0}^{*}$.  
\end{proof}

The use of square, sliding kernels of a fixed size for blurring during convolutions provides one example of a utilitarian preprocessing step performed during image segmentation.  Indeed, small regions of pixels are averaged in overlapping windows to smooth hard edges and to refine transitions between intensities in an image to reveal distinct objects.  Now, lems. $(\ref{probability_comparison_lemma},\ref{r_0_inequality})$ give further theoretical basis for the use of such a step.

Merely mapping higher dimensional data to a bounded region of $8$-bit integer representations (pixels) creates conditions so that disjoint regions of lower resolution are formed with higher likelihood, by lem. $(\ref{probability_comparison_lemma})$.  Then, by the hexagonal partition analog to thm. $(\ref{continuum_existence_r_0})$, we are (possibly) less likely to achieve the critical connection probability under an applied partition of the bounded region.  The "negative space" between fully-connected, disjoint regions, where closed edges exist, is then segmented and provides the basis for determining if distinct objects are present.

Likewise, by lem. $(\ref{r_0_inequality})$, the size of the disjoint regions are as large as possible, resulting in lowest resolution, in the presence of a partition defined by $r_{0}^{*}$.  Then, by the hexagonal partition analog to lem. $(\ref{continuum_probability_measure_non_decreasing_r})$, the probability is greatest for detection of open-to-closed edges lying between disjoint regions, as $\mathbf{P}(\mathcal{A}_{[n,\rho]}^{H_{r}})$ increases for $r \in (0,r_{0}^{*}]$, providing maximal segmentation probability at $r = r_{0}^{*}$ to (possibly) indicate distinct objects.

\subsection{Hexagonal Sharp Threshold Interval Length}

Given the particular radius guaranteed by thm. $(\ref{continuum_existence_r_0})$, then thm. $(\ref{geometric_graph_sharp_threshold_length})$ can be used to find an estimate of the length of the sharp threshold interval such that $\mathbf{P}(\mathcal{A}_{[n,\rho]}^{H_{r}})$ increases sharply from some $\epsilon \in (0,\frac{1}{2})$ to $1 - \epsilon$.  By lem. $(\ref{continuum_independence_r_0_epsilon})$, $r_{0}^{*}$ is independent of any particular $\epsilon$.  Thus, the interval and its length must be fixed, given $n$ and $\rho \in (\frac{1}{2},1)$.

\begin{theorem}
\label{hexagonal_partition_goel_sharp_threshold_length}
\begin{eqnarray}
\label{hexagonal_partition_goel_sharp_threshold_length_1} \Delta^{*}(n,\rho) & = & \Theta(r_{0}^{*}\log^{\frac{1}{4}}n) \\ & \ge & \label{hexagonal_partition_goel_sharp_threshold_length_2} \Delta(n,\rho).
\end{eqnarray}
with equality holding in ineq. $(\ref{hexagonal_partition_goel_sharp_threshold_length_2})$ if and only if $r_{0} = r_{0}^{*}$.
\end{theorem}

\begin{proof}
For $\delta \in (0,\frac{1}{2})$, let $\epsilon_{\delta} = \frac{1}{2} - \delta$.  By thm. $(\ref{geometric_graph_sharp_threshold_length})$ and thms. $(\ref{continuum_existence_r_0})$ and $(\ref{continuum_independence_r_0_epsilon})$,
\begin{eqnarray*}
\Delta^{*}(n,\rho) & = & \lim_{\delta \rightarrow 0^{+}} \Delta^{*}(n,\rho,\epsilon_{\delta}) \\ & = & \lim_{\delta \rightarrow 0^{+}} \Theta(r^{*}(n,\rho,\epsilon_{\delta})\log^{\frac{1}{4}}n) \\ & = & \Theta(r_{0}^{*}\log^{\frac{1}{4}}n).
\end{eqnarray*}  
This establishes eq. $(\ref{hexagonal_partition_goel_sharp_threshold_length_1})$.  The full ineq. $(\ref{hexagonal_partition_goel_sharp_threshold_length_2})$ is established by lem. $(\ref{r_0_inequality})$ and thm. $(\ref{geometric_graph_sharp_threshold_length})$, with eq. $(\ref{hexagonal_partition_goel_sharp_threshold_length_2})$ by cors. $(\ref{probability_equality_infinite},\ref{radial_equality})$, which will be stated and proven later.
\end{proof}

As in thm. $(\ref{continuum_goel_sharp_threshold_length})$ above, thm. $(\ref{hexagonal_partition_goel_sharp_threshold_length})$ gives an expected result, given thm. $(\ref{geometric_graph_sharp_threshold_length})$ above.  Likewise, a similar result to $\cite[Thm.\ (3.3.1)]{Cai}$ can be stated and later proven, as in the case of thm. $(\ref{continuum_probability_upper_lower_bound})$.  It is the result of thm. $(\ref{hexagonal_partition_probability_upper_lower_bound})$ that allows us to estimate the length of the sharp threshold interval in the presence of the hexagonal partition of $\mathcal{B}$.

What remains of this section is the main goal of completely framing the imaging problem, and by extension, the general segmentation problem, in the equivalent hexagonal formulation, requiring that $r_{0} = r_{0}^{*}$, as in proof of thm. $(\ref{hexagonal_partition_goel_sharp_threshold_length})$.  To this end, we introduce another notion of connectivity, that of connected partition structures.  As such, we treat a structure in the partition as we would a point generated by the point process, whenever the structure is occupied by at least one generated point.  Then, connectivity, and the notion of open/closed edges between structures in the partition, is the same as that between points in the original hexagonal formulation.  As such, we make no further distinction and proceed with the analysis.  Finally, the rest of this section is a treatise on estimating the probability of the hexagonal graph property and its extension to infinite graphs, obtained by permutations of all orders of structures in the partition.  This extension removes boundary effects and allows us to define conditions under which $r_{0} = r_{0}^{*}$.

\begin{theorem}
\label{hexagonal_partition_probability_upper_lower_bound}
There is a constant $c > 0$, independent of $M$, such that for all $\epsilon_{1} > 0$ and every fixed small $\delta > 0$
\begin{eqnarray*}
\mathbf{P}(\mathcal{A}_{[n,\rho+\delta]}^{H_{r}}) \leq (\frac{1}{2} + \epsilon_{1})M^{-c(r_{0}^{*} - r)}
\end{eqnarray*}
for all $r \leq r_{0}^{*}$ and
\begin{eqnarray}
\label{hexagonal_partition_probability_upper_lower_bound_2}
\mathbf{P}(\mathcal{A}_{[n,\rho-\delta]}^{H_{r}}) \geq 1 - (\frac{1}{2} + \epsilon_{1})M^{-c(r - r_{0}^{*})}
\end{eqnarray}
for all $r \geq r_{0}^{*}$.
\end{theorem}

Let $M^{2}$ be the number of hexagons partitioning the region $\mathcal{B}$ and let $H_{\mathcal{B}}(r) = H_{r} \cap \mathcal{B}$.  Given $<C>_{H_{r}} \ \subseteq \mathcal{X}_{n}$, define $H_{C} = \{h_{\mathcal{B}}^{r} \in H_{\mathcal{B}}(r) : h_{\mathcal{B}}^{r} \ \cap <C>_{H_{r}} \ \neq \emptyset\}$ to be the connected cluster of hexagons such that each hexagon contains at least one point from the connected cluster of points, $<C>_{H_{r}}$.

\begin{lemma}
$\mathbf{E}[\ \rho_{n}(C)\ ] = \frac{\mathbf{E}[\ |H_{C}|\ ]}{M^{2}}$.
\end{lemma}

\begin{proof}
Let $<C>_{H_{r}} \ \subseteq \mathcal{X}_{n}$ be an $H_{r}$-connected cluster and let $K_{H_{C}}$ be a random variable taking as values the number of points in the region $R_{H_{C}}$ defined by the hexagons in $H_{C}$.  Since the $n$ points are uniformly distributed spatially and $\mathcal{B}$ is partitioned into $M^{2}$ copies of the prototypical hexagon $h^{r}$, then
\begin{eqnarray*}
\mathbf{E}[\ K_{H_{C}}\ ] & = & n\frac{\mathbf{E}[\ area(R_{H_{C}})\ ]}{area{(\mathcal{B})}} \\ & = & n\frac{\mathbf{E}[\ |H_{C}|\ ] \times area(h^{r})}{M^{2} \times area(h^{r})} \\ & = & n\frac{\mathbf{E}[\ |H_{C}|\ ]}{M^{2}}.
\end{eqnarray*}
But, $\mathbf{E}[\ K_{H_{C}}\ ] = \mathbf{E}[\ |<C>_{H_{r}}|\ ]$.  Therefore,
\begin{eqnarray*}
\mathbf{E}[\ |<C>_{H_{r}}|\ ] = n\frac{\mathbf{E}[\ |H_{C}|\ ]}{M^{2}}
\end{eqnarray*}
implies
\begin{eqnarray*}
\mathbf{E}[\ \rho_{n}(C)\ ] = \frac{\mathbf{E}[\ |H_{C}|\ ]}{M^{2}}.
\end{eqnarray*}  
\end{proof}

Define $\mathcal{D}_{[n,\rho]}^{r} = \{H_{C} \subseteq H_{\mathcal{B}}(r) : \mathbf{E}[\ \rho_{n}(C)\ ] \geq \rho\}$.  With $\mathcal{D}_{[n,\rho]}^{r}$ defined as such, the original problem of estimating the length of the sharp threshold for the property $\mathcal{A}_{[n,\rho]}^{r}$ in the continuum is now recast as a site percolation problem on a hexagonal lattice.  As will be defined later, a site in the lattice will be deemed open if the corresponding hexagon is occupied by at least one of the points from $\mathcal{X}_{n}$ and it will be deemed closed otherwise.  Likewise, two sites are connected and belong to the same connected cluster if both sites are open and their hamming distance is less than or equal to one.

Later, a torus on the lattice will be formed by defining a countable collection of permutations of the hexagons in the partition so that the length of the sharp threshold for the property $\mathcal{D}_{[n,\rho]}^{r}$ can be approximated by the length for another property $\hat{\mathcal{D}}_{[n,\rho]}^{r}$ on the torus.  In this way, boundary connection issues for sites in the partition of $\mathcal{B}$ are mitigated and the length of the sharp threshold interval for the property $\hat{\mathcal{D}}_{[n,\rho]}^{r}$ approximates the length for $\mathcal{D}_{[n,\rho]}^{r}$, which approximates the length for $\mathcal{A}_{[n,\rho]}^{H_{r}}$, which finally approximates the length for $\mathcal{A}_{[n,\rho]}^{r}$, the original property in the continuum.

\begin{theorem}
\label{hexagons_probability_upper_lower_bound}
There is a constant $c > 0$, independent of $M$, such that
\begin{eqnarray*}
\mathbf{P}(\mathcal{D}_{[n,\rho]}^{r}) \leq \frac{1}{2}M^{-c(r_{0}^{*} - r)}
\end{eqnarray*}
for all $r \leq r_{0}^{*}$.  Similarly, for some fixed small $\delta > 0$ and for all $\epsilon_{1} > 0$, there is an $M_{0}(\delta,\epsilon_{1})$ such that for all $M > M_{0}(\delta,\epsilon_{1})$
\begin{eqnarray*}
\mathbf{P}(\mathcal{D}_{[n,\rho-\delta]}^{r}) \geq 1 - (\frac{1}{2} + \epsilon_{1})M^{-c(r - r_{0}^{*})}
\end{eqnarray*}
for all $r \geq r_{0}^{*}$.
\end{theorem}

An important part of the proof of thm. $(\ref{hexagons_probability_upper_lower_bound})$ relies upon the sharp threshold inequality results of $\cite{Bourgain}$ and $\cite{Friedgut}$.  To apply these results, connectivity in the hexagon lattice structure should be extended to the case of a torus, whereby any boundary connectivity issues are mitigated.  As such, make $H_{\mathcal{B}}(r)$ into a torus by identifying $h_{i,j} \in H_{\mathcal{B}}(r)$ with an element $h_{i^{\prime},j^{\prime}}$ in a copy of $H_{\mathcal{B}}(r)$, if $i^{\prime} = i \bmod M$ and $j^{\prime} = j \bmod M$.

For every $k,l \in \mathbb{Z}$, the mapping $\tau_{k,l} : h_{i,j} \rightarrow h_{i+k,j+l}$ defines a shift translation.  In this way, a subgroup of automorphisms $\tau = \{\tau_{k,l} : k,l \in \mathbb{Z}\}$ with the transitivity property is formed.  Thus, any hexagon $h_{i,j}$ can be shifted to any other hexagon $h_{i^{\prime},j^{\prime}}$ with the translation, $\tau_{i^{\prime}-i,j^{\prime}-j}$.  Hexagons in the 1st row (column) are allowed to be joined in a connected cluster with hexagons in the Mth row (column), provided all hexagons in question are occupied.

\begin{proposition}
\label{torus_containment_proposition}
Define $\tau(H_{\mathcal{B}}(r))$ to be the torus created by translations of hexagons in $H_{\mathcal{B}}(r)$ under the action of permutations in $\tau$ and define $\hat{\mathcal{D}}_{[n,\rho]}^{r} = \{H_{C} \subseteq \tau(H_{\mathcal{B}}(r)) : \mathbf{E}[\ \rho_{n}(C)\ ] \geq \rho\}$.  Then, $\mathcal{D}_{[n,\rho]}^{r} \subset \hat{\mathcal{D}}_{[n,\rho]}^{r}$ and $\mathcal{D}_{[n,\rho]}^{r} \neq \hat{\mathcal{D}}_{[n,\rho]}^{r}$.
\end{proposition}

\begin{proof}
Since $\hat{\mathcal{D}}_{[n,\rho]}^{r}$ contains all of the connected hexagons from $\mathcal{D}_{[n,\rho]}^{r}$ and any connections between the 1st and Mth rows (columns) while $\mathcal{D}_{[n,\rho]}^{r}$ contains no connection between the 1st and Mth rows (columns), then the result follows.  
\end{proof}

\begin{definition}
To each hexagon in the partition of $\mathcal{B}$, associate a \textit{site} $i \in \{1,2,...,M^{2}\}$ as the center of the hexagon.  For sites $i \in \{1,2,...,M^{2}\}$, define $s_{i} \in \{-1,+1\}$ to be the \textit{state on site $i$}.  A site $i$ is said to be \textit{open} if $s_{i} = +1$ and \textit{closed} otherwise.  There exists an \textit{edge} $e_{\{i,j\}}$ between sites $i,j \in \{1,2,...,M^{2}\}$ if and only if there exists a hexagon $h_{i,j}^{r} \ni i,j$ or there exists neighboring hexagons $h_{i}^{r} \ni i$ and $h_{j}^{r} \ni j$ in the partition of $\mathcal{B}$.  Define $e_{\{i,j\}}$ to be \textit{open} if and only if $s_{i} = 1 = s_{j}$ and \textit{closed} otherwise.
\end{definition}

\begin{definition}
The \textit{conditional influence of i} on the property $\hat{\mathcal{D}}_{[n,\rho]}^{r}$ is defined to be
\begin{eqnarray*}
I(i) = \mathbf{P}(\hat{\mathcal{D}}_{[n,\rho]}^{r} \ | \ s_{i} = +1) - \mathbf{P}(\hat{\mathcal{D}}_{[n,\rho]}^{r} \ | \ s_{i} = -1)
\end{eqnarray*}
and it is a measure of the change in the probability of $\hat{\mathcal{D}}_{[n,\rho]}^{r}$ due to a state change from $s_{i} = -1$ to $s_{i} = +1$ at site, $i$.
\end{definition}

For completeness, $\cite[Lemma\ (4.1.1)]{Cai}$ is stated without proof, which gives an upper bound on the change in $\mathbf{P}(\hat{\mathcal{D}}_{[n,\rho]}^{r})$ as a function of the point density $\lambda$.  Utilizing the chain rule for derivatives, a lower bound on the change in $\mathbf{P}(\hat{\mathcal{D}}_{[n,\rho]}^{r})$ as a function of $r$ is found and the resulting inequality relationship is used to estimate upper and lower bounds on $\mathbf{P}(\hat{\mathcal{D}}_{[n,\rho]}^{r})$, which will approximate the inequality results of thm. $(\ref{hexagons_probability_upper_lower_bound})$.

\begin{lemma}
\textit{$\cite[Lemma\ (4.1.1)]{Cai}$}
\label{cai_probability_change_upper_bound}
There is a constant $z > 0$, independent of $M$ and $\lambda$, such that
\begin{eqnarray*}
\frac{d}{d\lambda}\mathbf{P}(\hat{\mathcal{D}}_{[n,\rho]}^{r}) \leq z^{*}(\lambda)\min\{\mathbf{P}(\hat{\mathcal{D}}_{[n,\rho]}^{r}),1-\mathbf{P}(\hat{\mathcal{D}}_{[n,\rho]}^{r})\}\log{M}
\end{eqnarray*}
where $A_{h^{r}}$ is the area of the prototypical hexagon $h^{r}$ and $z^{*}(\lambda) = -zA_{h^{r}}e^{-A_{h^{r}}\lambda}$.
\end{lemma}

\begin{lemma}
\label{hexagon_probability_change_upper_bound_lemma}
There is a constant $c > 0$, independent of $M$ and $\lambda$, such that
\begin{eqnarray*}
\frac{d}{dr}\mathbf{P}(\hat{\mathcal{D}}_{[n,\rho]}^{r}) \geq c^{*}(\lambda)\min\{\mathbf{P}(\hat{\mathcal{D}}_{[n,\rho]}^{r}),1-\mathbf{P}(\hat{\mathcal{D}}_{[n,\rho]}^{r})\}\log{M}
\end{eqnarray*}
where $A_{h^{r}}$ is the area of the prototypical hexagon $h^{r}$ and $c^{*}(\lambda) = c(\lambda)A_{h^{r}}e^{-A_{h^{r}}\lambda}$, with $c(\lambda) = -cg(\lambda)$ for some function $g(\lambda)$.
\end{lemma}

\begin{proof}
As in cor. $(\ref{continuum_independence_r_0_epsilon_corollary})$, let $n^{*}$ be the inverse of $r^{*}$ and seeking a contradiction, suppose $dr/d\lambda = 0$.  Let $\epsilon \in (0,\frac{1}{2})$.  By lem. $(\ref{cai_probability_change_upper_bound})$, $dP/d\lambda$ exists.  Now, the existence of $dP/dr$ will be shown by proving a Lipschitz condition on the probability distribution $\mathbf{P}(\hat{\mathcal{D}}_{[n,\rho]}^{r})$ as a function of $r$.  Assume area$(\mathcal{B}) = 1$.  Without loss of generality, it can be assumed that $r \in [0,1]$.  Without further loss of generality, let $r_{1}^{*},r_{2}^{*} \in [0,1]$ such that $r_{0}^{*}$ is the midpoint of $[r_{1}^{*},r_{2}^{*}]$, i.e. $r_{0}^{*} = (r_{2}^{*}-r_{1}^{*})/2$.  Then, by thm. $(\ref{hexagonal_partition_goel_sharp_threshold_length})$,
\begin{eqnarray*}
|\mathbf{P}(\hat{\mathcal{D}}_{[n,\rho]}^{r_{2}^{*}}) - \mathbf{P}(\hat{\mathcal{D}}_{[n,\rho]}^{r_{1}^{*}})| & \leq & 1 = (\Delta^{*}(n,\rho))^{-1}|r_{2}^{*} - r_{1}^{*}|.
\end{eqnarray*}
Therefore, $\mathbf{P}(\hat{\mathcal{D}}_{[n,\rho]}^{r})$ is Lipschitz continuous with respect to $r$.  Hence, $dP/dr$ exists.  Now, since $dP/d\lambda$, $dP/dr$ and $dr/d\lambda$ all exist, then the Chain Rule for derivatives yields,
\begin{eqnarray*}
\frac{d}{d\lambda}\mathbf{P}(\hat{\mathcal{D}}_{[n,\rho]}^{r}) = \frac{d}{dr}\mathbf{P}(\hat{\mathcal{D}}_{[n,\rho]}^{r}) \times \frac{dr}{d\lambda}.
\end{eqnarray*}
Note that the existence of $dP/dr$ requires that $|dP/dr| < \infty$.  Therefore, since $dr/d\lambda = 0$, then
\begin{eqnarray*}
\frac{d}{d\lambda}\mathbf{P}(\hat{\mathcal{D}}_{[n,\rho]}^{r}) = \frac{d}{dr}\mathbf{P}(\hat{\mathcal{D}}_{[n,\rho]}^{r}) \times 0 = 0.
\end{eqnarray*}
As a result, $\mathbf{P}(\hat{\mathcal{D}}_{[n,\rho]}^{r})$ is constant as a function of $\lambda$.  So, suppose that $0 < n < n^{*}$.  Then, $\mathbf{P}(\hat{\mathcal{D}}_{[n,\rho]}^{r}) = 0$, which implies that $\mathbf{P}(\hat{\mathcal{D}}_{[n,\rho]}^{r}) \equiv 0$ for $n > 0$, $\rho \in (\frac{1}{2},1)$ and all $r \in (0,1]$.  This is a contradiction, since $\mathbf{P}(\hat{\mathcal{D}}_{[n,\rho]}^{r}) = 1$ for $r = 1$.  Hence, $dr/d\lambda \neq 0$.  Now, by $\cite[Thm.\ (2.28)]{Grimmett2}$, there is a constant $c > 0$, independent of $M$ and $\lambda$, such that
\begin{eqnarray*}
I(i) \geq c\min\{\mathbf{P}(\hat{\mathcal{D}}_{[n,\rho]}^{r}),1-\mathbf{P}(\hat{\mathcal{D}}_{[n,\rho]}^{r})\}\frac{\log{M}}{M^{2}}.
\end{eqnarray*}
Under the action of $\tau$, each hexagon in the bounded region $\mathcal{B}$ is translated to another hexagon in a copy of $\mathcal{B}$.  Therefore, $\hat{\mathcal{D}}_{[n,\rho]}^{r}$ and $\mathbf{P}(\hat{\mathcal{D}}_{[n,\rho]}^{r})$ are invariant under the action of $\tau$.  Hence, $I(i) = I(j)$ whenever, $\tau(i) = j$, where $\tau(i)$ is defined to be the translation of the hexagon $h_{i}^{r} \ni i$ to the hexagon $h_{j}^{r} \ni j$ in the copy of the partition of $\mathcal{B}$.  From $\cite{Cai}$, in the proof of thm. $(\ref{cai_probability_change_upper_bound})$, the following identity holds, with $p = p(\lambda) = 1-e^{-A_{h^{r}}\lambda}$ defined above,
\begin{eqnarray}
\nonumber \frac{d}{d\lambda}\mathbf{P}(\hat{\mathcal{D}}_{[n,\rho]}^{r}) & = & \nonumber \frac{d}{dp}\mathbf{P}(\hat{\mathcal{D}}_{[n,\rho]}^{r})\times\frac{dp}{d\lambda} \\ & = & \label{sum_influences_inequality} -A_{h^{r}}e^{-A_{h^{r}}\lambda} \sum_{i=1}^{M^{2}}I(i).
\end{eqnarray}
For $\gamma > 0$, $r > 0$ and $k > 0$, any $H_{r}$-connected component in $\mathcal{X}_{n}$ containing at least $\gamma(n+k)/2$ points will inherently contain an $H_{r}$-connected component of size at least $\gamma n/2$.  Hence, $\mathcal{A}_{[\gamma(n+k),\rho]}^{H_{r}} \subseteq \mathcal{A}_{[\gamma n,\rho]}^{H_{r}}$.  It follows that $\mathbf{P}(\mathcal{A}_{[\gamma(n+k),\rho]}^{H_{r}}) \leq \mathbf{P}(\mathcal{A}_{[\gamma n,\rho]}^{H_{r}})$.  Therefore, $r^{*}(\gamma n,\rho,\epsilon) \in \{r > 0 : \mathbf{P}(\mathcal{A}_{[\gamma(n+k),\rho]}^{H_{r}}) \geq \epsilon\}$, which implies $r^{*}(\gamma(n+k),\rho,\epsilon) \leq r^{*}(\gamma n,\rho,\epsilon)$ for $k > 0$.  Hence,
\begin{eqnarray}
\label{radial_change_less_than_0}
r^{*}(\gamma(n+k),\rho,\epsilon) - r^{*}(\gamma n,\rho,\epsilon) \leq 0.
\end{eqnarray}
Since point density $\lambda$ is proportional to point count $n$ for any bounded region $\mathcal{B}$, then using ineq. $(\ref{radial_change_less_than_0})$ yields
\begin{eqnarray*}
\frac{dr}{d\lambda} = \lim_{k \rightarrow 0} \frac{r^{*}(\gamma(n+k),\rho,\epsilon)-r^{*}(\gamma n,\rho,\epsilon)}{\gamma k} \leq 0,
\end{eqnarray*}
for some $\gamma > 0$.  Since $dr/d\lambda \neq 0$, it follows that
\begin{eqnarray*}
\frac{dr}{d\lambda} < 0.
\end{eqnarray*}
Since $dr/d\lambda$ exists, then $|dr/d\lambda| < \infty$.  Thus, by substituting
\begin{eqnarray*}
I(i) \geq c\min\{\mathbf{P}(\hat{\mathcal{D}}_{[n,\rho]}^{r}),1-\mathbf{P}(\hat{\mathcal{D}}_{[n,\rho]}^{r})\}\frac{\log{M}}{M^{2}}
\end{eqnarray*}
into $(\ref{sum_influences_inequality})$, it follows that
\begin{eqnarray*}
\frac{d}{d\lambda}\mathbf{P}(\hat{\mathcal{D}}_{[n,\rho]}^{r}) & = & \nonumber -A_{h^{r}}e^{-A_{h^{r}}\lambda} \sum_{i=1}^{M^{2}}I(i) \\ & \leq & -cA_{h^{r}}e^{-A_{h^{r}}\lambda} \\ &\times& \sum_{i=1}^{M^{2}} \min\{\mathbf{P}(\hat{\mathcal{D}}_{[n,\rho]}^{r}),1-\mathbf{P}(\hat{\mathcal{D}}_{[n,\rho]}^{r})\}\frac{\log{M}}{M^{2}} \\ & = & -cA_{h^{r}}e^{-A_{h^{r}}\lambda} \\ &\times& \min\{\mathbf{P}(\hat{\mathcal{D}}_{[n,\rho]}^{r}),1-\mathbf{P}(\hat{\mathcal{D}}_{[n,\rho]}^{r})\}\log{M}.
\end{eqnarray*}
Therefore,
\begin{eqnarray}
\frac{d}{d\lambda}\mathbf{P}(\hat{\mathcal{D}}_{[n,\rho]}^{r}) & = & \nonumber \frac{d}{dr}\mathbf{P}(\hat{\mathcal{D}}_{[n,\rho]}^{r}) \times \frac{dr}{d\lambda} \nonumber \\ \label{probability_decreasing_inequality} & \leq & -cA_{h^{r}}e^{-A_{h^{r}}\lambda} \nonumber \\ &\times& \min\{\mathbf{P}(\hat{\mathcal{D}}_{[n,\rho]}^{r}),1-\mathbf{P}(\hat{\mathcal{D}}_{[n,\rho]}^{r})\} \nonumber \\ &\times& \nonumber \log{M}
\end{eqnarray}
so that
\begin{eqnarray}
\label{probability_increasing_inequality}
\frac{d}{dr}\mathbf{P}(\hat{\mathcal{D}}_{[n,\rho]}^{r}) &\geq& -cA_{h^{r}}e^{-A_{h^{r}}\lambda}\bigg(\frac{dr}{d\lambda}\bigg)^{-1} \nonumber \\ &\times& \min\{\mathbf{P}(\hat{\mathcal{D}}_{[n,\rho]}^{r}),1-\mathbf{P}(\hat{\mathcal{D}}_{[n,\rho]}^{r})\} \nonumber \\ &\times& \log{M}.
\end{eqnarray}
Defining $g(\lambda) = (dr/d\lambda)^{-1}$, the result follows.  
\end{proof}

\begin{remark}
Let $\epsilon > 0$ be given.  At the risk of ambiguity, denote $n = \mathbf{E}[n]$ and define $\lambda^{*}(n,\rho,\epsilon) = \inf{\left\{n > 0 \ | \ \mathbf{P}(\hat{\mathcal{D}}_{[n,\rho]}^{r}) \ge \epsilon\right\}}$.  Inequality $(\ref{probability_decreasing_inequality})$ implies that $\mathbf{P}(\hat{\mathcal{D}}_{[n,\rho]}^{r})$ is increasing as a function of decreasing point density $\lambda = \lambda^{*}(n,\rho,\epsilon)$ such that the event $\left\{\mathbf{P}(\hat{\mathcal{D}}_{[n,\rho]}^{r}) \ge \epsilon\right\}$ first occurs.  Likewise, since the maximum distance between connected points is inversely proportional to point density, then ineq. $(\ref{probability_increasing_inequality})$ implies that $\mathbf{P}(\hat{\mathcal{D}}_{[n,\rho]}^{r})$ is decreasing as a function of increasing maximum distance $r = r^{*}(n,\rho,\epsilon)$ between connected points such that the event $\left\{\mathbf{P}(\hat{\mathcal{D}}_{[n,\rho]}^{r}) \ge \epsilon\right\}$ first occurs.
\end{remark}

\begin{lemma}
Let $c > 0$ be as in thm. $(\ref{hexagon_probability_change_upper_bound_lemma})$.  Then, there exists $r_{0}^{*}$, independent of $M$, such that
\begin{eqnarray*}
\mathbf{P}(\hat{\mathcal{D}}_{[n,\rho]}^{r}) \leq \frac{1}{2}M^{-c(r_{0}^{*} - r)}
\end{eqnarray*}
for all $r \leq r_{0}^{*}$ and
\begin{eqnarray*}
\mathbf{P}(\hat{\mathcal{D}}_{[n,\rho]}^{r}) \geq 1 - \frac{1}{2}M^{-c(r - r_{0}^{*})}
\end{eqnarray*}
for all $r \geq r_{0}^{*}$.
\end{lemma}

\begin{proof}
Arguing as in the proof to thm. $(\ref{continuum_existence_r_0})$, there exists $r_{0}^{*}$ such that $\mathbf{P}(\hat{\mathcal{D}}_{[n,\rho]}^{r_{0}^{*}}) = \frac{1}{2}$.  Arguing similarly to cor. $(\ref{continuum_continuity_corollary})$, $\mathbf{P}(\hat{\mathcal{D}}_{[n,\rho]}^{r})$ is continuous in $r$.  Therefore, $\mathbf{P}(\hat{\mathcal{D}}_{[n,\rho]}^{r}) \leq 1 - \mathbf{P}(\hat{\mathcal{D}}_{[n,\rho]}^{r})$ for $r \leq r_{0}^{*}$ and $\mathbf{P}(\hat{\mathcal{D}}_{[n,\rho]}^{r}) \geq 1 - \mathbf{P}(\hat{\mathcal{D}}_{[n,\rho]}^{r})$ for $r \geq r_{0}^{*}$.  Thus, the result of lem. $(\ref{hexagon_probability_change_upper_bound_lemma})$ takes the form
\begin{eqnarray*}
\frac{d}{dr}\mathbf{P}(\hat{\mathcal{D}}_{[n,\rho]}^{r}) \geq c^{*}(\lambda)\mathbf{P}(\hat{\mathcal{D}}_{[n,\rho]}^{r}) \log{M}
\end{eqnarray*}
for $r \leq r_{0}^{*}$ and
\begin{eqnarray*}
\frac{d}{dr}\mathbf{P}(\hat{\mathcal{D}}_{[n,\rho]}^{r}) \geq c^{*}(\lambda)(1-\mathbf{P}(\hat{\mathcal{D}}_{[n,\rho]}^{r})) \log{M}
\end{eqnarray*}
for $r \geq r_{0}^{*}$.  The last two inequalities can be written
\begin{eqnarray*}
\frac{d}{dr}\log{\mathbf{P}(\hat{\mathcal{D}}_{[n,\rho]}^{r})} \geq c^{*}(\lambda) \log{M}
\end{eqnarray*}
for $r \leq r_{0}^{*}$ and
\begin{eqnarray*}
\frac{d}{dr}\log{(1-\mathbf{P}(\hat{\mathcal{D}}_{[n,\rho]}^{r}))} \leq -c^{*}(\lambda) \log{M}
\end{eqnarray*}
for $r \geq r_{0}^{*}$, respectively.  Consider $r \leq r_{0}^{*}$.  Both sides of
\begin{eqnarray*}
\frac{d}{dr}\log{\mathbf{P}(\hat{\mathcal{D}}_{[n,\rho]}^{r})} \geq c^{*}(\lambda) \log{M}
\end{eqnarray*}
are integrated in the direction of increasing point density since $\mathbf{P}(\hat{\mathcal{D}}_{[n,\rho]}^{r})$ decreases as a function of point density $\lambda$ by the proof to lem. $(\ref{hexagon_probability_change_upper_bound_lemma})$.  It was also shown that $dr/d\lambda < 0$, i.e. $r$ is decreasing as a function of point density.  Therefore, the integration limits for the interval $[r,r_{0}^{*}]$ are from $r_{0}^{*}$ to $r$.  Noting that the inequality is reversed for backward integration, the following is obtained for $c > 0$ and some $K_{1}(\lambda) \geq 0$,
\begin{eqnarray*}
\log{\mathbf{P}(\hat{\mathcal{D}}_{[n,\rho]}^{r})} \leq K_{1}(\lambda) \log{M^{c(r-r_{0}^{*})}}
\end{eqnarray*}
which can be rewritten as
\begin{eqnarray*}
\log{\mathbf{P}(\hat{\mathcal{D}}_{[n,\rho]}^{r})} \leq K_{1}(\lambda) \log{M^{-c(r_{0}^{*}-r)}}.
\end{eqnarray*}
This implies
\begin{eqnarray*}
\mathbf{P}(\hat{\mathcal{D}}_{[n,\rho]}^{r}) \leq K_{2}(\lambda) M^{-c(r_{0}^{*}-r)}
\end{eqnarray*}
for some $K_{2}(\lambda) \geq 0$.  Therefore, using the initial condition $\mathbf{P}(\hat{\mathcal{D}}_{[n,\rho]}^{r_{0}^{*}}) = \frac{1}{2}$ yields $K_{2}(\lambda) = \frac{1}{2}$.  Thus,
\begin{eqnarray*}
\mathbf{P}(\hat{\mathcal{D}}_{[n,\rho]}^{r}) \leq \frac{1}{2}M^{-c(r_{0}^{*}-r)}.
\end{eqnarray*}
Now, consider $r \geq r_{0}^{*}$.  Similary, both sides of
\begin{eqnarray*}
\frac{d}{dr}\log{(1-\mathbf{P}(\hat{\mathcal{D}}_{[n,\rho]}^{r}))} \leq -c^{*}(\lambda) \log{M}
\end{eqnarray*}
are integrated in the direction of increasing connection radii on $[r_{0}^{*},r]$ since $\mathbf{P}(\hat{\mathcal{D}}_{[n,\rho]}^{r})$ increases as a function of connection radii $r$ by the proof to lem. $(\ref{hexagon_probability_change_upper_bound_lemma})$.  Therefore, the integration limits are from $r_{0}^{*}$ to $r$.  The following is obtained for $c > 0$ and some $K_{3}(\lambda) \geq 0$,
\begin{eqnarray*}
\log{(1-\mathbf{P}(\hat{\mathcal{D}}_{[n,\rho]}^{r}))} \leq -K_{3}(\lambda) \log{M^{c(r-r_{0}^{*})}}
\end{eqnarray*}
which can be rewritten as
\begin{eqnarray*}
\log{(1-\mathbf{P}(\hat{\mathcal{D}}_{[n,\rho]}^{r}))} & \leq & -K_{3}(\lambda) \log{M^{-c(r_{0}^{*}-r)}} \\ & = & K_{3}(\lambda) \log{M^{-c(r-r_{0}^{*})}}.
\end{eqnarray*}
This implies
\begin{eqnarray*}
1-\mathbf{P}(\hat{\mathcal{D}}_{[n,\rho]}^{r}) \leq K_{4}(\lambda) M^{-c(r-r_{0}^{*})}
\end{eqnarray*}
for some $K_{4}(\lambda) \geq 0$.  Therefore, using the initial condition $\mathbf{P}(\hat{\mathcal{D}}_{[n,\rho]}^{r_{0}^{*}}) = \frac{1}{2}$ yields $K_{4}(\lambda) = \frac{1}{2}$.  Hence,
\begin{eqnarray*}
\mathbf{P}(\hat{\mathcal{D}}_{[n,\rho]}^{r}) \geq 1-\frac{1}{2}M^{-c(r-r_{0}^{*})}.
\end{eqnarray*}  
\end{proof}

By prop. $(\ref{torus_containment_proposition})$, there are cases when $\mathcal{D}_{[n,\rho]}^{r} \subset \hat{\mathcal{D}}_{[n,\rho]}^{r}$, but $\mathcal{D}_{[n,\rho]}^{r} \neq \hat{\mathcal{D}}_{[n,\rho]}^{r}$ so that the occurrence of $\hat{\mathcal{D}}_{[n,\rho]}^{r}$ does not imply the occurrence of $\mathcal{D}_{[n,\rho]}^{r}$.  To exclude these possibilities, the arguments of $\cite{Cai}$ are followed whereby a slightly larger property $\mathcal{D}_{[n,\rho-\delta]}^{r}$ is considered for some small $\delta > 0$ such that the occurrence of $\hat{\mathcal{D}}_{[n,\rho]}^{r}$ implies the occurrence of $\mathcal{D}_{[n,\rho-\delta]}^{r}$.

As in $\cite{Cai}$, let $\phi(M)$ be any $M$-dependent integer such that $\phi(M) \rightarrow \infty$ as $M \rightarrow \infty$ and
\begin{eqnarray*}
\phi(M) = o(c(r-r_{0}^{*})\log{M}).
\end{eqnarray*}
Choose a coordinate system so that $\mathcal{B}$ has its lower left corner at the origin.  Define the top, bottom, left and right boundary strips of $\mathcal{B}$ as $H_{i}, i = 1,2,3,4$ with sizes $\phi(M) \times M$, $\phi(M) \times M$, $M \times \phi(M)$ and $M \times \phi(M)$ by
\begin{eqnarray*}
H_{1} = \{H_{i,j} : i = M-\phi(M)+1,...,M, j = 1,...,M\}
\end{eqnarray*}
\begin{eqnarray*}
H_{2} = \{H_{i,j} : i = 1,...,\phi(M), j = 1,...,M\}
\end{eqnarray*}
\begin{eqnarray*}
H_{3} = \{H_{i,j} : i = 1,...,M, j = 1,...,\phi(M)\}
\end{eqnarray*}
\begin{eqnarray*}
H_{4} = \{H_{i,j} : i = 1,...,M, j = M-\phi(M)+1,...,M\}.
\end{eqnarray*}
Let $E_{i}$ be the event that there is a connected path of \textit{occupied} hexagons crossing the rectangle $H_{i}$ using the longest straight-line path.

\begin{lemma}
\label{strip_lemma}
For $i = 1,2,3,4$, there are constants $c_{i} > 0$ such that for large $M$ and $r \geq r_{0}^{*}$
\begin{eqnarray*}
\mathbf{P}(E_{i}) \geq 1 - e^{-c_{i}\phi(M)}.
\end{eqnarray*}
\end{lemma}

\begin{proof}
As in $\cite{Cai}$, by the duality property, the occurrence of $E_{i}, i = 1,2,3,4$ is equivalent to the non-occurrence of the event that there is a connected path of \textit{unoccupied} hexagons crossing $H_{i}, i = 1,2,3,4$ using the shortest straight-line path.  The rest of the proof follows $\cite{Cai}$ with the edge probability as a function of point density $p(\lambda_{0})$ replaced by $r^{*}(n,\rho,\epsilon)$ and the critical probability for the occurrence of an infinite cluster of occupied hexagons $p_{c}$ replaced by $r_{0}^{*}$.  
\end{proof}

\begin{proof}
\textit{(Theorem $\ref{hexagons_probability_upper_lower_bound}$)}
By prop. $(\ref{torus_containment_proposition})$, $\mathcal{D}_{[n,\rho]}^{r} \subset \hat{\mathcal{D}}_{[n,\rho]}^{r}$ so that $\mathbf{P}(\mathcal{D}_{[n,\rho]}^{r}) \leq \mathbf{P}(\hat{\mathcal{D}}_{[n,\rho]}^{r})$.  To estimate $\mathbf{P}(\mathcal{D}_{[n,\rho-\delta]}^{r})$ for $r > r_{0}$ and any given $\delta > 0$, let $E = E_{1} \cap E_{2} \cap E_{3} \cap E_{4}$ and consider $F = \hat{\mathcal{D}}_{[n,\rho]}^{r} \cap E$.  Since $\mathbf{P}(F) = \mathbf{P}(F \cap \mathcal{D}_{[n,\rho-\delta]}^{r}) + \mathbf{P}(F - \mathcal{D}_{[n,\rho-\delta]}^{r})$, then
\begin{eqnarray*}
\mathbf{P}(\mathcal{D}_{[n,\rho-\delta]}^{r}) \geq \mathbf{P}(F) - \mathbf{P}(F - \mathcal{D}_{[n,\rho-\delta]}^{r}).
\end{eqnarray*}
Noting that $\mathbf{P}(E_{1}) = \mathbf{P}(E_{2})$ and $\mathbf{P}(E_{3}) = \mathbf{P}(E_{4})$, then the FKG inequality of $\cite{Grimmett}$ yields
\begin{eqnarray*}
\mathbf{P}(F) \geq \mathbf{P}(\hat{\mathcal{D}}_{[n,\rho]}^{r})\mathbf{P}^{2}(E_{1})\mathbf{P}^{2}(E_{3}).
\end{eqnarray*}
By lem. $(\ref{strip_lemma})$, there exists $b > 0$ such that for all sufficiently large $M$,
\begin{eqnarray*}
\mathbf{P}(F) \geq 1 - \frac{1}{2}M^{-c(r-r_{0}^{*})} - O(e^{-b\phi(M)}).
\end{eqnarray*}
Using $\phi(M) = o(c(r-r_{0}^{*})\log{M})$, this implies that for any given $\epsilon_{1} > 0$ and all sufficiently large $M$ depending upon $\epsilon_{1}$,
\begin{eqnarray*}
\mathbf{P}(F) \geq 1 - \left(\frac{1}{2} + \epsilon_{1}\right)M^{-c(r-r_{0}^{*})}.
\end{eqnarray*}
It is now claimed that $F - \mathcal{D}_{[n,\rho-\delta]}^{r} = \emptyset$, requiring that $\mathbf{P}(F - \mathcal{D}_{[n,\rho-\delta]}^{r}) = 0$ for all large $M$.  Following $\cite{Cai}$, the occurrence of $F$ implies that there is a connected path of hexagons which encloses the sub-lattice given by $H_{\mathcal{B}}(r) - \bigcup_{i = 1}^{4} H_{i}$.  Because the points in $\mathcal{X}_{n}$ are uniformly distributed, then there is a connected cluster of hexagons within the original lattice totaling at least $\rho M^{2} - (2M\phi(M) + 2\phi(M)(M-2\phi(M)))$ hexagons, where $\rho M^{2}$ is a lower bound on the number of occupied hexagons in the largest connected cluster and $2M\phi(M) + 2\phi(M)(M-2\phi(M))$ is the total number of hexagons in the strips, $H_{i}, i = 1,2,3,4$.  Let $\delta_{1} = (2M\phi(M) + 2\phi(M)(M-2\phi(M)))/M^{2}$.  It follows that $F \subset \mathcal{D}_{[n,\rho-\delta_{1}]}^{r}$, since $F$ occurs in those hexagons of $\mathcal{B}$ that are not near the boundary of $\mathcal{B}$ by a simple translation $\tau$ of hexagons $h \in \bigcup_{i = 1}^{4} H_{i}$ to hexagons $h \in H_{\mathcal{B}}(r) - \bigcup_{i = 1}^{4} H_{i}$.  Thus, if $M$ is large enough so that $\delta_{1} < \delta$, then $F \subset \mathcal{D}_{[n,\rho-\delta_{1}]}^{r} \subset \mathcal{D}_{[n,\rho-\delta]}^{r}$.  
\end{proof}

\begin{proof}
\textit{(Theorem $\ref{hexagonal_partition_probability_upper_lower_bound}$)} Consider $r \leq r_{0}^{*}$.  Since
\begin{eqnarray*}
\mathbf{P}(\mathcal{A}_{[n,\rho+\delta]}^{H_{r}}) = \mathbf{P}(\mathcal{A}_{[n,\rho+\delta]}^{H_{r}},\mathcal{D}_{[n,\rho]}^{r}) + \mathbf{P}(\mathcal{A}_{[n,\rho+\delta]}^{H_{r}} - \mathcal{D}_{[n,\rho]}^{r})
\end{eqnarray*}
then
\begin{eqnarray*}
\mathbf{P}(\mathcal{A}_{[n,\rho+\delta]}^{H_{r}}) \leq \mathbf{P}(\mathcal{D}_{[n,\rho]}^{r}) + \mathbf{P}(\mathcal{A}_{[n,\rho+\delta]}^{H_{r}} - \mathcal{D}_{[n,\rho]}^{r}).
\end{eqnarray*}
It will be shown that $\mathbf{P}(\mathcal{A}_{[n,\rho+\delta]}^{H_{r}} - \mathcal{D}_{[n,\rho]}^{r}) = o(M^{-c(r_{0}^{*}-r)})$.  Let $x$ be a configuration of states across hexagons in $H_{\mathcal{B}}(r)$ and let $\mathcal{C}(x) = \{C_{1},...,C_{K}\}$ be the set of clusters in $x$.  For $i = 1,...,K$, let $N_{C_{i}}$ be the number of points in the cluster, $C_{i}$.  Then, $\{N_{C_{i}} \ | \ \mathcal{C}(x),n\} \sim B\left(n,\frac{|H_{C_{i}}|}{M^{2}}\right)$.  Suppose $C_{i_{0}} \in \mathcal{C}(x)$ is any cluster such that $\rho_{n}(C_{i_{0}}) \geq \rho + \delta$.  Since the occurrence of the property $\left(\mathcal{D}_{[n,\rho]}^{r}\right)^{c}$ implies $\frac{|H_{C_{i_{0}}}|}{M^{2}} < \rho$, then
\begin{eqnarray*}
\mathcal{A}_{[n,\rho+\delta]}^{H_{r}} - \mathcal{D}_{[n,\rho]}^{r} \subset \left\{\rho_{n}(C_{i_{0}}) \geq \rho+\delta, \frac{|H_{C_{i_{0}}}|}{M^{2}} < \rho\right\}.
\end{eqnarray*}
By arguments in $\cite{Durrett}$ and $\cite{Cai}$, there is an $\alpha = \alpha(\rho,\delta) > 0$ such that
\begin{eqnarray*}
\mathbf{P}\left(\rho_{n}(C_{i_{0}}) \geq \rho + \delta \ \middle| \ \left\{\frac{|H_{C_{i_{0}}}|}{M^{2}} < \rho\right\}, \mathcal{C}(x), n\right) \leq e^{-\alpha(\rho,\delta)n}.
\end{eqnarray*}
It follows that
\begin{eqnarray*}
\mathbf{P}(\mathcal{A}_{[n,\rho+\delta]}^{H_{r}} - \mathcal{D}_{[n,\rho]}^{r}) &\leq& \mathbf{P}\left(\{\rho_{n}(C_{i_{0}}) \geq \rho + \delta\}, \left\{\frac{|H_{C_{i_{0}}}|}{M^{2}} < \rho\right\}\right) \\ &\leq& \mathbf{P}\left(\rho_{n}(C_{i_{0}}) \geq \rho + \delta \ \middle| \ \frac{|H_{C_{i_{0}}}|}{M^{2}} < \rho\right) \times \mathbf{P}\left(\frac{|H_{C_{i_{0}}}|}{M^{2}} < \rho\right) \label{some_inequality_1} \\ &\leq& \mathbf{P}\left(\rho_{n}(C_{i_{0}}) \geq \rho + \delta \ \middle| \ \frac{|H_{C_{i_{0}}}|}{M^{2}} < \rho\right) \\ &=& \mathbf{E}\left[\mathbf{P}\left(\rho_{n}(C_{i_{0}}) \geq \rho + \delta \ \middle| \ \left\{\frac{|H_{C_{i_{0}}}|}{M^{2}} < \rho\right\},\mathcal{C}(x),n\right)\right] \\ &\leq& \mathbf{E}[e^{-\alpha n}] \label{some_inequality_2} \\ &=& \exp{\{-n(1-e^{-\alpha})\}}.
\end{eqnarray*}
Now, since $n(1-e^{-\alpha}) > d\log{M}$ implies $\exp{\{-n(1-e^{-\alpha})\}} < M^{-d}$, then for any $d > 0$ and every fixed $\delta > 0$, it follows that $\mathbf{P}(\mathcal{A}_{[n,\rho+\delta]}^{H_{r}} - \mathcal{D}_{[n,\rho]}^{r})$ decays to zero at a rate faster than $M^{-d}$ for $n$ large enough.  The case of $r \geq r_{0}^{*}$ is proven with similar arguments.  
\end{proof}

\begin{theorem}
\label{hexagonal_partition_continuity_rho}
$\mathbf{P}(\mathcal{A}_{[n,\rho]}^{H_{r}})$ is a continuous function of $\rho$.
\end{theorem}

\begin{proof}
Let $\sigma = 1 - \rho$ in eq. $(\ref{hexagonal_partition_event_finite})$.  Then, $\mathcal{A}_{[n,\sigma]}^{H_{r}}$ is an increasing property in $\sigma$ for increasing $\rho \in (\frac{1}{2},1)$.  Therefore, by $\cite[Thm.\ (2.48)]{Grimmett2}$, $\mathcal{A}_{[n,\sigma]}^{H_{r}}$ has a sharp threshold in $\sigma$, and hence, in $\rho$.  Thus, by $\cite[Ineq.\ (2.49)]{Grimmett2}$, $\mathbf{P}(\mathcal{A}_{[n,\rho]}^{H_{r}})$ is differentiable in $\rho$, which implies that $\mathbf{P}(\mathcal{A}_{[n,\rho]}^{H_{r}})$ is continuous as a function of $\rho$.  
\end{proof}

\begin{remark}
By thm. $(\ref{hexagonal_partition_continuity_rho})$, if $r_{1}^{*} < r_{0}^{*} < r_{2}^{*}$ and for some $\epsilon \in (0,\frac{1}{2})$, we have $\mathbf{P}(\mathcal{A}_{[n,\rho]}^{H_{r_{1}^{*}}}) = \epsilon$ and $\mathbf{P}(\mathcal{A}_{[n,\rho]}^{H_{r_{2}^{*}}}) = 1 - \epsilon$, then $r_{2}^{*} - r_{1}^{*}$ is an estimate of the sharp threshold interval length for the property, $\mathcal{A}_{[n,\rho]}^{H_{r}}$.  As it will be shown later, $r_{0} = r_{0}^{*}$ under certain conditions, in which case, if $r_{1} < r_{0} < r_{2}$, then $r_{2} - r_{1}$ serves the same purpose, by thm. $(\ref{continuum_continuity_rho})$.
\end{remark}

\begin{proof}
\textit{(Theorem $\ref{continuum_probability_upper_lower_bound}$)} Since $\mathbf{P}(\mathcal{A}_{[n,\rho]}^{r})$ and $\mathbf{P}(\mathcal{A}_{[n,\rho]}^{H_{r}})$ are continuous functions of $r$, then by thm. $(\ref{hexagonal_partition_probability_upper_lower_bound})$ and lem. $(\ref{r_0_inequality})$, for every $r \in (0,r_{0}]$ there exists $r^{\prime} \leq r$ such that
\begin{eqnarray}
\label{probability_inequality_delta_rho} \mathbf{P}(\mathcal{A}_{[n,\rho+\delta]}^{r^{\prime}}) & \leq & \mathbf{P}(\mathcal{A}_{[n,\rho+\delta]}^{H_{r}}) \\ \nonumber & \leq & (\frac{1}{2} + \epsilon_{1})M^{-c(r_{0}^{*} - r)} \\ \label{probability_inequality_delta_rho_upper} & \leq & (\frac{1}{2} + \epsilon_{1})M^{-c(r_{0} - r)}.
\end{eqnarray}
Consider $r_{0} \in (0,r_{0}]$.  Then, continuity of $\mathbf{P}(\mathcal{A}_{[n,\rho]}^{r})$ in $r$ and the non-decreasing property of $\mathbf{P}(\mathcal{A}_{[n,\rho]}^{r})$ in $r$ implies ineq. $(\ref{probability_inequality_delta_rho_upper})$ for all $r \in [0,r^{\prime}]$.  It is claimed that $r^{\prime} = r_{0}$.  Seeking a contradiction if $r^{\prime} < r_{0}$, suppose $\mathbf{P}(\mathcal{A}_{[n,\rho]}^{r}) \leq (\frac{1}{2} + \epsilon_{1})M^{-c(r_{0} - r)}$ for all $r \in [0,r^{\prime}]$ and $\mathbf{P}(\mathcal{A}_{[n,\rho]}^{r}) > (\frac{1}{2} + \epsilon_{1})M^{-c(r_{0} - r)}$ for all $r > r^{\prime}$.  By hypothesis, $r_{0} > r^{\prime}$ so that when $r = r_{0}$, it follows that $\mathbf{P}(\mathcal{A}_{[n,\rho+\delta]}^{r_{0}}) > \frac{1}{2}$.  Now, since for any connected cluster $<C>_{r}$ such that $\rho_{n}(C) \geq \rho + \delta$ for $\delta > 0$, the statement $\rho_{n}(C) \geq \rho$ is implied, then $\mathcal{A}_{[n,\rho+\delta]}^{r} \subseteq \mathcal{A}_{[n,\rho]}^{r}$ for all $r \in (0,r_{0}]$.  Hence, $r^{\prime} < r_{0}$ leads to
\begin{eqnarray}
\label{probability_inequality_delta_rho_upper_out}
\mathbf{P}(\mathcal{A}_{[n,\rho]}^{r_{0}}) \geq \limsup_{\delta \rightarrow 0^{+}} \mathbf{P}(\mathcal{A}_{[n,\rho+\delta]}^{r_{0}}) \geq \mathbf{P}(\mathcal{A}_{[n,\rho+\delta]}^{r_{0}}) > \frac{1}{2}.
\end{eqnarray}
In particular, ineq. $(\ref{probability_inequality_delta_rho_upper_out})$ gives $\mathbf{P}(\mathcal{A}_{[n,\rho]}^{r_{0}}) > \frac{1}{2}$.  This is a contradiction since $\mathbf{P}(\mathcal{A}_{[n,\rho]}^{r_{0}}) = \frac{1}{2}$ by thm. $(\ref{continuum_existence_r_0})$.  It follows that $r^{\prime} = r_{0}$ and
\begin{eqnarray*}
\mathbf{P}(\mathcal{A}_{[n,\rho]}^{r}) \leq (\frac{1}{2} + \epsilon_{1})M^{-c(r_{0} - r)}
\end{eqnarray*}
for $r \leq r_{0}$.  A similar argument is used to prove
\begin{eqnarray*}
\mathbf{P}(\mathcal{A}_{[n,\rho-\delta]}^{r}) \geq 1 - (\frac{1}{2} + \epsilon_{1})M^{-c(r - r_{0})}
\end{eqnarray*}
for $r \geq r_{0}$.  
\end{proof}

The implication of the proof to thm. $(\ref{continuum_probability_upper_lower_bound})$ is that $\mathbf{P}(\mathcal{A}_{[n,\rho]}^{r}) = \mathbf{P}(\mathcal{A}_{[n,\rho]}^{H_{r}})$ for $r \in (0,r_{0}]$.  By $\cite[Thm.\ (1.16)]{Grimmett2}$, the random cluster measure gives rise to a collection of conditional probability measures of connection events in the identified clusters.  Therefore, the point process $X$ samples from each element of the collection.

\begin{theorem}
\label{probability_equality}
$\mathbf{P}(\mathcal{A}_{[n,\rho]}^{r}) = \mathbf{P}(\mathcal{A}_{[n,\rho]}^{H_{r}})$ for $r \in (0,r_{0}]$.
\end{theorem}

\begin{proof}
By continuity in $\rho$ of $\mathbf{P}(\mathcal{A}_{[n,\rho]}^{H_{r}})$ as given by thm. $(\ref{hexagonal_partition_continuity_rho})$, we have
\begin{eqnarray*}
\lim_{\delta \rightarrow 0^{+}} \mathbf{P}(\mathcal{A}_{[n,\rho+\delta]}^{H_{r}}) = \mathbf{P}(\mathcal{A}_{[n,\rho]}^{H_{r}}).
\end{eqnarray*}
Suppose $\delta_{1} > \delta_{2}$ such that $\rho+\delta_{1},\rho+\delta_{2} \in (\frac{1}{2},1)$ and let $<C>_{r} \ \in \mathcal{A}_{[n,\rho+\delta_{1}]}^{r}$.  Then, $\rho_{n}(C) \geq \rho+\delta_{1} > \rho+\delta_{2}$ so that $<C>_{r} \ \in \mathcal{A}_{[n,\rho+\delta_{2}]}^{r}$.  Hence, $\mathcal{A}_{[n,\rho+\delta_{1}]}^{r} \subseteq \mathcal{A}_{[n,\rho+\delta_{2}]}^{r}$.  By properties of probability measures, $\mathbf{P}(\mathcal{A}_{[n,\rho]}^{r})$ is monotone non-decreasing as a function of decreasing $\rho$.  By ineq. $(\ref{probability_inequality_delta_rho})$, it follows that for some fixed $r \in (0,r_{0}]$, there exists $r^{\prime} \in (0,r_{0}]$ such that $\mathbf{P}(\mathcal{A}_{[n,\rho+\delta]}^{r^{\prime}}) \leq \mathbf{P}(\mathcal{A}_{[n,\rho+\delta]}^{H_{r}})$ for all $r^{\prime\prime} \in [0,r^{\prime}]$ so that
\begin{eqnarray}
\limsup_{\delta \rightarrow 0^{+}} \mathbf{P}(\mathcal{A}_{[n,\rho+\delta]}^{r^{\prime\prime}}) \leq \limsup_{\delta \rightarrow 0^{+}} \mathbf{P}(\mathcal{A}_{[n,\rho+\delta]}^{H_{r}}) \label{inequality_lim_sup} = \mathbf{P}(\mathcal{A}_{[n,\rho]}^{H_{r}}).
\end{eqnarray}
From the proof of thm. $(\ref{continuum_probability_upper_lower_bound})$, it was shown that $r^{\prime} = r_{0}$.  Therefore, by continuity of $\mathbf{P}(\mathcal{A}_{[n,\rho]}^{H_{r}})$ in $r$, ineq. $(\ref{inequality_lim_sup})$ holds for all $r \in (0,r_{0}]$, with $r^{\prime\prime}$ replaced by $r$.  The Monotone Convergence Theorem $\cite{Schechter}$ applied to $\mathbf{E}[1_{\mathcal{A}_{[n,\rho+\delta]}^{r}}]$ and $\mathbf{E}[1_{\mathcal{A}_{[n,\rho]}^{r}}]$ guarantees that $\mathbf{P}(\mathcal{A}_{[n,\rho+\delta]}^{r}) \rightarrow \mathbf{P}(\mathcal{A}_{[n,\rho]}^{r})$ as $\delta \rightarrow 0^{+}$.  Therefore, ineq. $(\ref{inequality_lim_sup})$ becomes
\begin{eqnarray}
\label{inequality_lim_sup_2}
\mathbf{P}(\mathcal{A}_{[n,\rho]}^{r}) = \limsup_{\delta \rightarrow 0^{+}} \mathbf{P}(\mathcal{A}_{[n,\rho+\delta]}^{r}) \leq \mathbf{P}(\mathcal{A}_{[n,\rho]}^{H_{r}}).
\end{eqnarray}
In particular, $\mathbf{P}(\mathcal{A}_{[n,\rho]}^{r}) \leq \mathbf{P}(\mathcal{A}_{[n,\rho]}^{H_{r}})$ so that with the result of lem. $(\ref{probability_comparison_lemma})$, namely $\mathbf{P}(\mathcal{A}_{[n,\rho]}^{H_{r}}) \leq \mathbf{P}(\mathcal{A}_{[n,\rho]}^{r})$, the theorem follows.  
\end{proof}

\begin{corollary}
\label{probability_equality_infinite}
$\mathbf{P}(\mathcal{A}^{r}) = \mathbf{P}(\mathcal{A}^{H_{r}})$ for $r \in (0,r_{0}]$.
\end{corollary}

\begin{proof}
By thm. $(\ref{probability_equality})$, $\mathbf{P}(\mathcal{A}_{[n,\rho]}^{r}) = \mathbf{P}(\mathcal{A}_{[n,\rho]}^{H_{r}})$ for all $r \in (0,r_{0}]$ and all $n \geq 1$.  By prop. $(\ref{continuum_probability_measure_non_increasing_n})$, it follows that $\mathbf{P}(\mathcal{A}^{H_{r}}) \leq \mathbf{P}(\mathcal{A}_{[n,\rho]}^{H_{r}}) = \mathbf{P}(\mathcal{A}_{[n,\rho]}^{r})$.  In particular, $\mathbf{P}(\mathcal{A}^{H_{r}}) \leq \mathbf{P}(\mathcal{A}_{[n,\rho]}^{r})$.  Without loss of generality, assume that area$(\mathcal{B}) = 1$.  From $\cite{Cai}$, differentiability of $\mathbf{P}(\mathcal{A}_{[n,\rho]}^{r})$ in point density $\lambda = \lambda(n) = \mathbf{E}[n]$ implies continuity of $\mathbf{P}(\mathcal{A}_{[n,\rho]}^{r})$ in $\lambda$ so that the following holds
\begin{eqnarray}
\label{probability_limit}
\lim_{\mathbf{E}[n] \rightarrow \infty} \mathbf{P}(\mathcal{A}_{[n,\rho]}^{r}) = \mathbf{P}(\mathcal{A}^{r}).
\end{eqnarray}
Therefore, $\mathbf{P}(\mathcal{A}^{H_{r}}) \leq \mathbf{P}(\mathcal{A}_{[n,\rho]}^{r})$ and eq. $(\ref{probability_limit})$ implies $\mathbf{P}(\mathcal{A}^{H_{r}}) \leq \mathbf{P}(\mathcal{A}^{r})$.  Similarly, $\mathbf{P}(\mathcal{A}^{r}) \leq \mathbf{P}(\mathcal{A}^{H_{r}})$ so that the corollary follows.  
\end{proof}

\begin{corollary}
\label{radial_equality}
$r_{0} = r_{0}^{*}$.
\end{corollary}

\begin{proof}
By thm. $(\ref{probability_equality})$, $\frac{1}{2} = \mathbf{P}(\mathcal{A}_{[n,\rho]}^{r_{0}}) = \mathbf{P}(\mathcal{A}_{[n,\rho]}^{H_{r_{0}}})$.  In particular, $\frac{1}{2} = \mathbf{P}(\mathcal{A}_{[n,\rho]}^{H_{r_{0}}})$.  Since $\mathbf{P}(\mathcal{A}_{[n,\rho]}^{H_{r_{0}^{*}}}) = \frac{1}{2} = \mathbf{P}(\mathcal{A}_{[n,\rho]}^{H_{r_{0}}})$, by the discussion preceding thm. $(\ref{hexagonal_partition_goel_sharp_threshold_length})$ and by thm. $(\ref{continuum_existence_r_0})$, then the uniqueness of $r_{0}^{*}$ and $r_{0}$ guarantees that $r_{0}^{*} = r_{0}$.  
\end{proof}

\begin{proof}
\textit{(Theorem $\ref{continuum_continuity_rho}$)}
Follows directly from thms. $(\ref{hexagonal_partition_continuity_rho})$ and $(\ref{probability_equality})$.  
\end{proof}

By thm. $(\ref{probability_equality})$ and cor. $(\ref{radial_equality})$, the problem of estimating the probabilities and length of the sharp threshold interval in the continuum can be re-cast as problems of estimation in the presence of a hexagonal partition of the bounded region.  As such, tools from percolation $\cite{Grimmett}$ and the random cluster model $\cite{Grimmett2}$ can readily be employed.

Since $\mathbf{P}(\mathcal{A}_{[n,\rho]}^{r}) = \mathbf{P}(\mathcal{A}_{[n,\rho]}^{H_{r}})$ for $r \in (0,r_{0}]$ by cor. $(\ref{probability_equality_infinite})$, disjoint clusters of points in the continuum are equivalent to disjoint clusters of occupied hexagons in the hexagonal partition of the bounded region containing all points.  Multi-dimensional points in the continuum can be thought to belong to the same cluster if they are within a certain Euclidean distance of one another.  As a result, the multi-dimensional points will have representatives belonging to occupied, connected hexagons in the 2-dimensional, bounded, partitioned region.  We now have theoretical justification for the consideration of an image as a projection of higher dimensional data to $2$-dimensional space.

\section{Generalized Image: Sharp Threshold and Critical Radius Calculations}
\label{km}

Recall that an image can be viewed as the result of distortions of $2$-dimensional data after a loss of information from higher dimensional observations.  Suppose a bounded region $\mathcal{B} \subset \mathbb{R}^{2}$ is a  generalized image of data points.  The idea is to partition $\mathcal{B}$ into $M^{2}$ regions of hexagons and find $K = N^{2}$ contiguous clusters of hexagons such that each of the clusters are mutually disjoint.  Into one hexagon of exactly one cluster will a (higher dimensional) data point be mapped.

\begin{theorem}
\label{minimum_hexagon_theorem}
Assume that there are $M^{2}$ points and $N^{2}$ segmented clusters for the points. The minimum number of hexagons required to partition the unit square in $\mathbb{R}^{2}$ centered at the origin into $N^{2}$ disjoint regions such that $M^{2}$ is the sum total of all hexagons in the disjoint regions is given by
\begin{eqnarray*}
S(M,N) = M^{2} + 2M(N-1)^{2}.
\end{eqnarray*}
\end{theorem}

\begin{proof}
Since $M^{2} >> N^{2}$ by hypothesis, then the total number of hexagons required to partition $\mathcal{B}$ into disjoint regions of contiguous hexagons is $O(M^{2})$.  Label the disjoint regions $A_{1},A_{2},...,A_{N^{2}}$ and let $k$ be any integer such that $1 \leq k \leq N^{2}$.  Since the total number of hexagons partitioning $\mathcal{B}$ is $O(M^{2})$, then the number of hexagons in $A_{k}$ is proportional to $M^{2}$.  Likewise, the total number of hexagons in boundary$(A_{k})$ is proportional to area$(A_{k})$.  Since area$(A_{k})$ is proportional to $M^{2}$, then the number of hexagons in boundary$(A_{k})$ is proportional to $M^{2}$.  Note that each $A_{k}$ shares a portion of its separating boundary with each of its neighboring clusters of hexagons.  Let $A_{j}$ be a neighboring cluster of $A_{k}$ such that $j \neq k$ and $1 \leq j \leq N^{2}$.  Since this portion of the separating boundary is proportional to both area$(A_{k})$ and area$(A_{j})$, then it is proportional to a common area of size area$(A_{kj})$.  Repeating this same logic for all integers $k$ and $j$ such that $1 \leq k \leq N^{2}$ and $1 \leq j \leq N^{2}$, the total number of hexagons in the entire separating boundaries is proportional to a common area of size area$(A)$.  Since minimizing the total number of hexagons in $\mathcal{B}$ is tantamount to minimizing the area$(A)$, then making an application of the law of large numbers, each of the $N^{2}$ disjoint clusters of connected hexagons is the same size and must be a square sub-region of $\mathcal{B}$ containing $M^{2}/N^{2}$ hexagons.  The minimum number of hexagons that are required to enclose $N^{2}$ sub-regions of $\mathcal{B}$ containing $M^{2}/N^{2}$ hexagons is exactly $2M(N-1)^{2}$.  Therefore, the minimum number of hexagons required to partition $\mathcal{B}$ into $N^{2}$ disjoint regions such that $M^{2}$ is the sum total of all hexagons in the disjoint regions is given by
\begin{eqnarray}
\label{hexcount}
S(M,N) = M^{2} + 2M(N-1)^{2}.
\end{eqnarray}  
\end{proof}

The idea is to use the result of the theorem to calculate, as a function of $M$ and $N = N(M)$, the exact size of a prototypical hexagon which will be used to partition $\mathcal{B}$ into hexagons of equal size.  As $K = N^{2}$ is fixed as the number of clusters of data points, $M^{2}$ is fixed for the initial calculation of $S(M,N)$ and the subsequent segmenting of the first $M^{2}$ data points.

In $\cite{Grimmett}$, it is stated and proven that there is a critical probability of connection between hexagons containing a point of a network such that it is no longer possible to have disjoint clusters of points, when this critical probability of connection is exceeded.  Hence, all points will be connected into one cluster, which is not what we intend to model, in this case.

Since the size of $\mathcal{B}$ is fixed, then to decrease the probability of connection while maintaining $K = N^{2}$ disjoint contiguous clusters of points, the size of each hexagon must decrease while increasing the number of hexagons in the boundaries of the disjoint regions.  In this way, the ratio of the total number of occupied hexagons to the total number of hexagons will be less than this critical probability of connection.  Note that we used uniformity of the points throughout $\mathcal{B}$ so that the approximate number of points in a cluster of hexagons is proportional to the ratio of the number of hexagons in the cluster divided by the number of hexagons in the entire region, $\mathcal{B}$.  Also, note that the minimum number of hexagons required for separation is given by thm. $(\ref{minimum_hexagon_theorem})$, so that the common radius of the circle that can circumscribe any one of these hexagons is of size
\begin{equation}
\label{radius}
R(M,N) = \frac{1}{2\sqrt{S(M,N)}},
\end{equation}
thereby necessarily indicating that
\begin{equation*}
\label{diameter}
B(M,N) = 2*R(M,N)
\end{equation*}
is the diameter of the circumscribing circle.

\begin{lemma}
\label{Rincreasing}
$R(M,N)$ is decreasing for increasing $M$ and $N$.
\end{lemma}

\begin{proof}
By eq. $(\ref{hexcount})$, $S(M,N)$ is increasing for increasing $M$ and $N$.  Consequently, by eq. $(\ref{radius})$, $R(M,N)$ is decreasing for increasing $M$ and $N$.  
\end{proof}

\begin{theorem}
\label{infinite}
$\cite[Thm.\ (1.11)]{Grimmett}$
Suppose the point process $X$ generates infinitely many points in $\mathbb{R}^{2}$.  An infinite connected cluster exists across hexagons in $\mathbb{R}^{2}$ with probability $1$ if and only if the probability that any two points connect exceeds $p_{c}$, where $p_{c}$ is the critical probability of connection.  Otherwise, all connected clusters are disjoint with probability $1$.
\end{theorem}

A direct result of thm. $(\ref{infinite})$ is that, given any bounded region $\mathcal{B}$, all points generated within $\mathcal{B}$ are almost surely connected into one cluster, when $p_{c}$ is exceeded.  Therefore, in order to not exceed $p_{c}$, which means maintaining the $N^{2}$ clusters of $M^{2}$ data points, the radial length of each hexagon's circumscribing circle must be less than or equal to $R(M,N)$.  By $\cite[Thm.\ (1.11)]{Grimmett}$, the clusters will be disjoint with probability $1$.  Hence, cor. $(\ref{minimum_hexagon_corollary})$ follows from these statements and lem. $(\ref{some_lemma})$.

\begin{corollary}
\label{minimum_hexagon_corollary}
[to thm. $(\ref{minimum_hexagon_theorem})$]
Let $h^{r}$ be a hexagon of size such that it can be inscribed into a circle of radius $r = r(M,N) > 0$ where 
\begin{eqnarray*}
0 < r \leq R(M,N).
\end{eqnarray*}
If $\mathcal{B}$ is partitioned into copies of $h^{r}$, then with probability $1$, $N^{2}$ is the mean number of disjoint clusters of contiguous hexagons in the region $\mathcal{B}$ that are occupied by the $M^{2}$ points.
\end{corollary}

With $r_{0}=R(M,N)$ given by cor. $(\ref{minimum_hexagon_corollary})$, the size of the prototypical hexagon can be calculated for repartitioning $\mathcal{B}$.  Furthermore, cor. $(\ref{minimum_hexagon_corollary})$ guarantees that the clusters will remain distinct, with probability $1$, through each new segmentation.  By cor. $(\ref{minimum_hexagon_corollary})$, the expected value of the (minimum) number of clusters to form can be calculated.
\begin{lemma}
\label{some_lemma}
For $M^{2}$ uniformly distributed data points in $\mathcal{B}$ and for any $\rho \in (0,p_{c}]$, with $p_{c} = 1 - 2\sin{(\pi/18)}$,
\begin{eqnarray}
\label{nsquare}
\frac{M^{2}}{S(M,N)} = \frac{M^{2}}{M^{2} + 2M(N - 1)^{2}} = \rho
\end{eqnarray}
determines the expected number $K = N^{2}$ of disjoint clusters to form such that $M^{2}$ is the total of all occupied hexagons across all clusters.
\end{lemma}

\begin{proof}
At the risk of ambiguity, let $N^{2}$ denote both the random variable and the expectation of the random variable which takes the number of formed clusters as its value.  Because $\mathcal{B}$ is partitioned by hexagons, it is shown in $\cite[Chapter\ 3]{Grimmett2}$ that $p_{c} = 1 - 2\sin{(\pi/18)}$.  By uniformity, the mean number of data points in each cluster is $M^{2}/N^{2}$.  By thm. $(\ref{infinite})$, each cluster will be disjoint and each hexagon in $\mathcal{B}$ will be as large as possible if $\mathcal{B}$ is partitioned into $S(M,N)$ hexagons of equal size.  Also, by thm. $(\ref{infinite})$, the probability of any of the $M^{2}$ hexagons being populated with a data point has to be less than or equal to $p_{c}$ in order that the expected clusters form with probability $1$, resulting in eq. $(\ref{nsquare})$.  For any $\rho \in (0,p_{c}]$, $K = N^{2}$ is found by solving eq. $(\ref{nsquare})$ to obtain $K = N^{2}$ as the least integer which is not less than the integer part of a non-negative solution to eq. $(\ref{nsquare})$, for fixed, positive $M^{2}$.  
\end{proof}

\begin{lemma}
\label{complement_containment}
For fixed $M > 0$, $N > 1$ given by a solution to eq. $(\ref{nsquare})$ for $\rho \in (\frac{1}{2},1)$, and $r > 0$ there exists $\delta = \delta(\rho) \in (0,\frac{1}{2})$, such that
\begin{eqnarray*}
\left\{\frac{|<C>_{H_{r}}|}{S(M,N)} < \frac{1}{2}\right\} = \left(\mathcal{A}_{[S(M,N),\rho-\delta]}^{H_{r}}\right)^{c}
\end{eqnarray*}
upto sets of $\mathbf{P}$-measure zero.
\end{lemma}

\begin{proof}
By definition, $\left(\mathcal{A}_{[S(M,N),\rho-\delta]}^{H_{r}}\right)^{c} = \left\{\frac{|<C>_{H_{r}}|}{S(M,N)} < \rho-\delta\right\}$.  Take $\delta = \rho-\frac{1}{2}$.  
\end{proof}

For ease of notation, we define $\mathbf{A}_{r} := \left\{\frac{|<C>_{H_{r}}|}{S(M,N)} < \frac{1}{2}\right\} = \left(\mathcal{A}_{[S(M,N),\rho-\delta]}^{H_{r}}\right)^{c}$.  By choosing $\delta$ as in lem. $(\ref{complement_containment})$, continuity in $r > 0$ and the non-decreasing property of $\mathbf{P}(\mathbf{A}_{r}^{c})$ for increasing $r > 0$ granted by cor. $(\ref{continuum_continuity_corollary})$ and prop. $(\ref{continuum_probability_measure_non_decreasing_r})$, respectively, then by ineq. $(\ref{continuum_probability_upper_lower_bound_2})$, it follows that
\begin{eqnarray*}
R(M,N) < r_{0}^{*} = r_{0}^{*}(M,N)
\end{eqnarray*}
for the property $\mathbf{A}_{r}$, since
\begin{eqnarray*}
\mathbf{P}(\mathbf{A}_{R(M,N)}) & = & 1 \\ & > & \frac{1}{2} \\ & = & \mathbf{P}(\mathbf{A}_{r_{0}^{*}})
\end{eqnarray*}
and since the probability of $\mathbf{A}_{r}$ is non-decreasing for decreasing $r \leq r_{0}^{*}$, a reversal.

Let $\epsilon \in (0,\frac{1}{2})$ be given and let $r_{1}^{*} > 0$ and $r_{2}^{*} > 0$, guaranteed by cor. $(\ref{continuum_continuity_corollary})$, be such that $\mathbf{P}(\mathbf{A}_{r_{1}^{*}}) = 1 - \epsilon$ and $\mathbf{P}(\mathbf{A}_{r_{2}^{*}}) = \epsilon$, respectively.  Then, again by cor. $(\ref{continuum_continuity_corollary})$, it follows that
\begin{eqnarray*}
R(M,N) < r_{1}^{*} < r_{0}^{*} = r_{0}^{*}(M,N) < r_{2}^{*}.
\end{eqnarray*}
By symmetry, it follows that
\begin{eqnarray}
\label{interval_estimate}
R(M,N) < r_{1}^{*} < r_{0}^{*} = r_{0}^{*}(M,N) < r_{2}^{*} < 2r_{0}^{*} - R(M,N).
\end{eqnarray}
Note that by cor. $(\ref{minimum_hexagon_corollary})$ and by symmetry,
\begin{eqnarray*}
\mathbf{P}(\mathbf{A}_{r}) = 0
\end{eqnarray*}
when $r \geq 2r_{0}^{*} - R(M,N)$.  Therefore, if $\mathbf{A}_{r}$ occurs with probability $0$, then the property $\left\{\frac{M^{2}}{S(M,N)} < \frac{1}{2}\right\}$ occurs with probability $0$.  Otherwise, $\mathbf{A}_{r}$ would occur with positive probability, since $\left\{\frac{M^{2}}{S(M,N)} < \frac{1}{2}\right\} \subseteq \left\{\frac{|<C>_{H_{r}}|}{S(M,N)} < \frac{1}{2}\right\} = \mathbf{A}_{r}$, upto sets of $\mathbf{P}$-measure zero, by lem. $(\ref{complement_containment})$.  Hence, $\left\{\frac{M^{2}}{S(M,N)} \geq \frac{1}{2}\right\}$ occurs with probability 1.  As a result,
\begin{eqnarray}
\label{nsol}
\frac{M^{2}}{M^{2} + 2M(N-1)^{2}} \geq \frac{1}{2}
\end{eqnarray}
with probability $1$.  Therefore, with probability $1$ for $M$, it follows that $N$ is a (positive integer) solution to $(N-1)^{2} \le 2M$.

\begin{lemma}
\label{rN_lemma}
If $r \geq \frac{1}{2N}$, then $\mathbf{P}(\mathbf{A}_{r}) = 0$.
\end{lemma}

\begin{proof}
Without loss of generality, suppose area$(\mathcal{B}) = 1$ and further suppose that $\mathcal{B}$ is divided into squares with sides of length $2r = \frac{1}{N}$.  By hypothesis, $\mathcal{B}$ contains $M^{2}$ data points and it is to be divided into $N^{2}$ regions.  Clearly then, there are no boundary hexagons separating each of the $N^{2}$ regions since the sides of $\mathcal{B}$ have length $2rN = 1$ which gives $\mathcal{B}$ an area of $1$.  Let each square be inscribed by a circle of radius $r$, which itself is inscribed by a hexagon.  By hypothesis, each of the $N^{2}$ regions in $\mathcal{B}$ contains at least one of the $M^{2}$ data points.  Hence, each of the $N^{2}$ (occupied) regions is connected in a cluster to every other region in $\mathcal{B}$ so that $\mathbf{P}(\mathbf{A}_{r}^{c}) = 1$.  Since $\mathbf{P}(\mathbf{A}_{r}^{c}) = 1$ for $r = \frac{1}{2N}$, then $\mathbf{P}(\mathbf{A}_{r}^{c}) = 1$ for $r \geq \frac{1}{2N}$ by prop. $(\ref{continuum_probability_measure_non_decreasing_r})$.  
\end{proof}

As a result of lem. $(\ref{rN_lemma})$ and by using ineq. $(\ref{interval_estimate})$, a conservative estimate for $r_{0}^{*}$ is given by a solution to
\begin{eqnarray}
\label{estimate_r_0}
2r_{0}^{*} - R(M,N) \geq \frac{1}{2N}
\end{eqnarray}
that maximizes $\frac{1}{2N}$ as a function of $M$.  The value of $N$ satisfies ineq. $(\ref{nsol})$ and a maximal solution is found when equality holds.  As such, for $\epsilon \in (0,\frac{1}{2})$, since $(r_{1}^{*},r_{2}^{*}) \subset (\ R(M,N),2r_{0}^{*} - R(M,N)\ )$, then by ineq. $(\ref{estimate_r_0})$,
\begin{eqnarray}
\nonumber r_{2}^{*} - r_{1}^{*} & \approx & 2r_{0}^{*} - 2R(M,N) \\ \label{interval_length} & = & \frac{1}{2N} - R(M,N)
\end{eqnarray}
is an estimate of the length of the sharp threshold interval $r_{2}^{*}-r_{1}^{*}$ about $r_{0}^{*}$.

Using the value of $r_{0}^{*}$ given by eq. $(\ref{estimate_r_0})$ and by using the estimate for the length of the sharp threshold interval about $r_{0}^{*}$ given by eq. $(\ref{interval_length})$, an estimate for the value of $r_{1}^{*}$ can be obtained.  Thus, when $r \leq r_{1}^{*}$, the property $\mathbf{A}_{r}$ occurs with probability at least $1 - \epsilon$ and falls sharply to a probability of occurrence of $\epsilon$ as $r \rightarrow r_{2}^{*}$.

By cor. $(\ref{radial_equality})$ and thm. $(\ref{probability_equality})$, the left half of the sharp threshold interval about $r_{0}$ is given by $[r_{1}^{*},r_{0}]$.  Using lem. $(\ref{probability_comparison_lemma})$, there exists $r_{2} \le r_{2}^{*}$ such that $[r_{0},r_{2}]$ is the right half of the sharp threshold interval for $\epsilon > 0$ given.  Thus, when $r \leq r_{1}^{*}$, the property $\mathbf{A}_{r}$ occurs with probability at least $1 - \epsilon$ and falls sharply to a probability of occurrence of (no greater than) $\epsilon$ as $r \rightarrow r_{2}^{*}$.  As such, the sharp threshold interval for clustering $M^{2}$ data points into $N^{2}$ clusters, in the mean continuum case, is of length (no greater than) $r_{2}^{*}-r_{1}^{*}$.

\begin{theorem}
Let $\Delta^{*}(M,N)$ denote the sharp threshold interval length for the event of segmenting $M^{2}$ random data points into $N^{2}$ clusters.  Then,
\begin{eqnarray*}
\Delta^{*}(M,N) = O(N^{-1}).
\end{eqnarray*}
\end{theorem}

\begin{proof}
Follows directly from eq. $(\ref{interval_length})$, eq. $(\ref{radius})$ and thm. $(\ref{minimum_hexagon_theorem})$.  
\end{proof}

\section{Conclusions}

It was shown that by bijectively mapping (possibly) higher dimensional data into a partitioned $2$-dimensional space, a critical radius of connectivity could be found such that when radii are less than the critical value, then data can be segmented into (at least) a minimum number of clusters.  The result holds for images, which are shown to justifiably be a special case of the bijection, indicating a minimum number of segmented objects in what are called "interesting" images.  Under multiple formulations, the length of a sharp threshold interval was estimated, upon which, the general case of randomly-generated data points almost surely form connected edges in a single connected cluster and in which the special case of images transition to an uninteresting, singly-colored background object, almost surely.

\section{Appendix}

\subsection{Graph}

\begin{proposition}
\label{continuum_graph_containment}
If $r < r^{\prime}$, then $G(\mathcal{X}_{n};r) \subseteq G(\mathcal{X}_{n};r^{\prime})$.
\end{proposition}

\begin{proof}
Suppose $r < r^{\prime}$.  If $<x,y>_{r} \ \in G(\mathcal{X}_{n};r)$, then $d(x,y) \leq r < r^{\prime}$ so that $<x,y>_{r} \ \in G(\mathcal{X}_{n};r^{\prime})$.  Hence, $G(\mathcal{X}_{n};r) \subseteq G(\mathcal{X}_{n};r^{\prime})$.  
\end{proof}

\subsection{Increasing Property}

\begin{lemma}
\label{continuum_increasing_property_uniqueness}
$|\mathcal{A}_{[n,\rho]}^{r}| \leq 1$.
\end{lemma}

\begin{proof}
If $\mathcal{A}_{[n,\rho]}^{r} = \emptyset$, then there is nothing to prove.  Thus, suppose that $\mathcal{A}_{[n,\rho]}^{r}$ occurs and $<C>_{r} \ \in \mathcal{A}_{[n,\rho]}^{r}$.  Since $\rho_{n}(C) \geq \rho > \frac{1}{2}$, then all other connected components are of order strictly less than half of all points.  Therefore, $|\mathcal{A}_{[n,\rho]}^{r}| = 1$.  
\end{proof}

\begin{proposition}
\label{continuum_increasing_property_r}
$\mathcal{A}_{[n,\rho]}^{r}$ is an increasing property in $r$.
\end{proposition}

\begin{proof}
Suppose $<C>_{r} \ \in \mathcal{A}_{[n,\rho]}^{r}$ and fix arbitrary $r^{\prime} > r$.  Then, $d(x,y) \leq r < r^{\prime}$ for all $x,y \in \ <C>_{r}$.  Thus, $<C>_{r} \ \subseteq \ <C>_{r^{\prime}}$ implies $N = |<C>_{r}| \leq |<C>_{r^{\prime}}|$.  Hence, $<C>_{r} \ \in \mathcal{A}_{[n,\rho]}^{r}$ implies $<C>_{r^{\prime}} \ \in \mathcal{A}_{[n,\rho]}^{r}$.  Since $r^{\prime} > r$ is arbitrary, then $\mathcal{A}_{[n,\rho]}^{r}$ is an increasing property in $r$.  
\end{proof}

\begin{proposition}
\label{continuum_decreasing_property_n}
$\mathcal{A}_{[n,\rho]}^{r}$ is a decreasing property in $n$.
\end{proposition}

\begin{proof}
Suppose $<C>_{r} \ \in \mathcal{A}_{[n,\rho]}^{r}$.  If $n^{\prime} < n$, then $|<C>_{r}|/n^{\prime} > |<C>_{r}|/n \geq \rho$ so that $<C>_{r} \in \mathcal{A}_{[n^{\prime},\rho]}^{r}$.  Hence, $\mathcal{A}_{[n,\rho]}^{r} \subseteq \mathcal{A}_{[n^{\prime},\rho]}^{r}$.  Since $n^{\prime} < n$, then $\mathcal{A}_{[n,\rho]}^{r}$ is decreasing in $n$.  
\end{proof}

\subsection{Probability Measure}

\begin{proposition}
\label{continuum_probability_measure_measurable}
The property $\mathcal{A}_{[n,\rho]}^{r}$ is $\mathbf{P}$-measurable.
\end{proposition}

\begin{proof}
For $x,y \in \mathcal{X}_{n}$ and $S \subseteq \mathcal{X}_{n}$, define the state on $<x,y>_{r}$ to be $1$ if and only if $<x,y>_{r} \ \in G(S;r)$ and $-1$ otherwise.  Then, $S$ mutually determines an element $\omega_{S} \in \Omega = \{-1,1\}^{\mathcal{X}_{n}}$ so that $S$ is $\mathbf{P}$-measureable.  Since $\mathcal{A}_{[n,\rho]}^{r}$ is the property that there exists $\omega_{S} \in \Omega$ mutually determined by $S \subseteq \mathcal{X}_{n}$ such that $(\max_{y \in S}|<C_{y}>_{r}|)/n \geq \rho$, then $\mathcal{A}_{[n,\rho]}^{r}$ is $\mathbf{P}$-measureable.  
\end{proof}

\begin{proposition}
\label{continuum_probability_measure_non_decreasing_r}
$\mathbf{P}(\mathcal{A}_{[n,\rho]}^{r})$ is a non-decreasing function of $r$.
\end{proposition}

\begin{proof}
Suppose $r_{1} \leq r_{2}$.  Since $\mathcal{A}_{[n,\rho]}^{r}$ is an increasing property in $r$ by prop. $(\ref{continuum_increasing_property_r})$, then $\mathcal{A}_{[n,\rho]}^{r_{1}} \subseteq \mathcal{A}_{[n,\rho]}^{r_{2}}$ so that $\mathbf{P}(\mathcal{A}_{[n,\rho]}^{r_{1}}) \leq \mathbf{P}(\mathcal{A}_{[n,\rho]}^{r_{2}})$ by properties of probability measures.  Thus, $\mathbf{P}(\mathcal{A}_{[n,\rho]}^{r})$ is non-decreasing in $r$.  
\end{proof}

\begin{proposition}
\label{continuum_probability_measure_non_increasing_n}
$\mathbf{P}(\mathcal{A}_{[n,\rho]}^{r})$ is a non-increasing function of $n$.
\end{proposition}

\begin{proof}
Suppose $n^{\prime} < n$.  Since $\mathcal{A}_{[n,\rho]}^{r}$ is a decreasing property in $n$ by prop. $(\ref{continuum_decreasing_property_n})$, then $\mathcal{A}_{[n,\rho]}^{r} \subseteq \mathcal{A}_{[n^{\prime},\rho]}^{r}$ so that $\mathbf{P}(\mathcal{A}_{[n,\rho]}^{r}) \leq \mathbf{P}(\mathcal{A}_{[n^{\prime},\rho]}^{r})$ by properties of probability measures.  Thus, $\mathbf{P}(\mathcal{A}_{[n,\rho]}^{r})$ is non-increasing in $n$.  
\end{proof}

\subsection{Connection Radius}

\begin{proposition}
\label{continuum_connection_radius_non_decreasing_epsilon}
$r(n,\rho,\epsilon)$ is a non-decreasing function of $\epsilon$.
\end{proposition}

\begin{proof}
Suppose $\epsilon_{1},\epsilon_{2} \in (0,\frac{1}{2})$ such that $\epsilon_{1} \leq \epsilon_{2}$.  Define $r_{1} = r(n,\rho,\epsilon_{1})$ and $r_{2} = r(n,\rho,\epsilon_{2})$ and suppose $r_{1} > r_{2}$.  Since $\mathbf{P}(\mathcal{A}_{[n,\rho]}^{r})$ is non-decreasing in $r$ by prop. $(\ref{continuum_probability_measure_non_decreasing_r})$, then $\mathbf{P}(\mathcal{A}_{[n,\rho]}^{r_{1}}) \geq \mathbf{P}(\mathcal{A}_{[n,\rho]}^{r_{2}}) \geq \epsilon_{2} \geq \epsilon_{1}$.  Hence, $r_{2} \in \{r > 0 : \mathbf{P}(\mathcal{A}_{[n,\rho]}^{r}) \geq \epsilon_{1}\}$ and $r_{2} < r_{1} = \inf \{r > 0 : \mathbf{P}(\mathcal{A}_{[n,\rho]}^{r}) \geq \epsilon_{1}\}$.  Contradiction.  Thus, $r_{1} \leq r_{2}$ so that $r(n,\rho,\epsilon)$ is non-decreasing in $\epsilon$.  
\end{proof}

\begin{lemma}
\label{continuum_connection_radius_X_n}
If $R = 2 * \max \{d(x,y) : x,y \in \mathcal{X}_{n}\}$, then $\mathcal{X}_{n} = \{x \in \mathcal{X}_{n} : d(x,y) \leq R\}$ for all fixed $y \in \mathcal{X}_{n}$.
\end{lemma}

\begin{proof}
Clearly, $\{x \in \mathcal{X}_{n} : d(x,y) \leq R\} \subseteq \mathcal{X}_{n}$.  Conversely, fix any $y \in \mathcal{X}_{n}$.  For every $x \in \mathcal{X}_{n}$, $d(x,y) \leq 2 * \max \{d(x,y) : x,y \in \mathcal{X}_{n}\} = R$.  Hence, $\mathcal{X}_{n} \subseteq \{x \in \mathcal{X}_{n} : d(x,y) \leq R\}$ for all fixed $y \in \mathcal{X}_{n}$.  Thus, $\mathcal{X}_{n} = \{x \in \mathcal{X}_{n} : d(x,y) \leq R\}$ for all fixed $y \in \mathcal{X}_{n}$.  
\end{proof}

\begin{corollary}
\label{continuum_connection_radius_C}
If $R = 2 * \max \{d(x,y) : x,y \in \mathcal{X}_{n}\}$, then $<C_{y}>_{R} \ \in \mathcal{A}_{[n,\rho]}^{R}$ for all $y \in \mathcal{X}_{n}$ and $n \geq 1$.
\end{corollary}

\begin{proof}
Fix an arbitrary $y \in \mathcal{X}_{n}$.  By lem. $(\ref{continuum_connection_radius_X_n})$, if $<C_{y}>_{R} \ = \{x \in \mathcal{X}_{n} : d(x,y) \leq R\}$, then $<C_{y}>_{R} \ = \mathcal{X}_{n}$ so that $|<C_{y}>_{R}| \ = |\mathcal{X}_{n}| = n$.  Therefore, since $y \in \mathcal{X}_{n}$ is arbitrary, then $<C_{y}>_{R} \ \in \mathcal{A}_{[n,\rho]}^{R}$ for all $y \in \mathcal{X}_{n}$ and $n \geq 1$.  
\end{proof}

\begin{corollary}
\label{continuum_connection_radius_P_1}
If $R = 2 * \max \{d(x,y) : x,y \in \mathcal{X}_{n}\}$, then $\mathbf{P}(\mathcal{A}_{[n,\rho]}^{R}) = 1$ for all $n \geq 1$.
\end{corollary}

\begin{proof}
By lem. $(\ref{continuum_connection_radius_X_n})$ and cor. $(\ref{continuum_connection_radius_C})$, $\mathcal{X}_{n} \in \mathcal{A}_{[n,\rho]}^{R}$ for all $n \geq 1$ and $\rho \in (\frac{1}{2},1)$.  Thus, $\mathcal{A}_{[n,\rho]}^{R} \neq \emptyset$ for all $n \geq 1$ and $\rho \in (\frac{1}{2},1)$.  Hence, $\mathbf{P}(\mathcal{A}_{[n,\rho]}^{R}) = 1$ for all $n \geq 1$.  
\end{proof}

\begin{lemma}
\label{continuum_connection_radius_bound_R}
If $R = 2 * \max \{d(x,y) : x,y \in \mathcal{X}_{n}\}$, then $0 < r(n,\rho,\epsilon) \leq R$ for all $\epsilon \in (0,\frac{1}{2})$.
\end{lemma}

\begin{proof}
By lem. $(\ref{continuum_connection_radius_X_n})$, $\mathcal{X}_{n} = \{x \in \mathcal{X}_{n} : d(x,y) \leq R\}$ for all fixed $y \in \mathcal{X}_{n}$.  Therefore, $\mathbf{P}(\mathcal{A}_{[n,\rho]}^{R}) = 1 \geq \epsilon$ for all $\epsilon \in (0,\frac{1}{2})$.  Suppose that $\epsilon_{0} \in (0,\frac{1}{2})$ exists such that $r_{0} = r(n,\rho,\epsilon_{0}) > R$.  Thus, $\mathcal{A}_{[n,\rho]}^{R} \subseteq \mathcal{A}_{[n,\rho]}^{r_{0}}$ so that
\begin{eqnarray*}
1 = \mathbf{P}(\mathcal{A}_{[n,\rho]}^{R}) \leq \mathbf{P}(\mathcal{A}_{[n,\rho]}^{r_{0}})
\end{eqnarray*}
since $\mathbf{P}(\mathcal{A}_{[n,\rho]}^{r})$ is non-increasing in $n$ by prop. $(\ref{continuum_connection_radius_C})$, non-decreasing in $r$ by prop. $(\ref{continuum_probability_measure_non_decreasing_r})$ and by properties of probability measures.  Hence, $\mathbf{P}(\mathcal{A}_{[n,\rho]}^{r_{0}}) = 1$.  But, then $R \in \{r > 0 : \mathbf{P}(\mathcal{A}_{[n,\rho]}^{r}) \geq \epsilon_{0}\}$ and $R < r_{0} = \inf \{r > 0 : \mathbf{P}(\mathcal{A}_{[n,\rho]}^{r}) \geq \epsilon_{0}\}$.  Contradiction.  Thus, $0 < r_{0} \leq R$.  Therefore, $0 < r(n,\rho,\epsilon) \leq R$ for all $\epsilon \in (0,\frac{1}{2})$.  
\end{proof}

\begin{proposition}
\label{continuum_connection_radius_convergent_sequence}
Suppose $\{\epsilon_{k} \in (0,\frac{1}{2})\}_{k \geq 1}$ is any convergent sequence such that $\epsilon_{k} \rightarrow \epsilon_{0}$.   Define $r_{k} = r(n,\rho,\epsilon_{k})$ and $r_{0} = r(n,\rho,\epsilon_{0})$.  For arbitrary $\xi > 0$,  if $\{k \geq 1 : |\mathbf{P}(\mathcal{A}_{[n,\rho]}^{r_{k}}) - \mathbf{P}(\mathcal{A}_{[n,\rho]}^{r_{0}})| \geq \xi\}$ is a set of measure zero, then $r_{k} \rightarrow r_{0}$ as $k \rightarrow \infty$.
\end{proposition}

\begin{proof}
If $\xi > 0$ is arbitrary and $\{k \geq 1 : |\mathbf{P}(\mathcal{A}_{[n,\rho]}^{r_{k}}) - \mathbf{P}(\mathcal{A}_{[n,\rho]}^{r_{0}})| \geq \xi\}$ is a set of measure zero, then
\begin{eqnarray*}
\mathbf{P}(\mathcal{A}_{[n,\rho]}^{r_{k}}) = \mathbf{P}(\mathcal{A}_{[n,\rho]}^{r_{0}}) \geq \epsilon_{0}
\end{eqnarray*}
for all $k \geq 1$.  Hence, $r_{k} \in \{r > 0 : \mathbf{P}(\mathcal{A}_{[n,\rho]}^{r}) \geq \epsilon_{0}\}$ for all $k \geq 1$.  Thus,
\begin{eqnarray}
\nonumber \lim_{k \rightarrow \infty} r_{k} & = & \lim_{k \rightarrow \infty} r(n,\rho,\epsilon_{k}) \\ \label{continuum_limit_inf_1} & = & \lim_{k \rightarrow \infty} \inf \{r > 0 : \mathbf{P}(\mathcal{A}_{[n,\rho]}^{r}) \geq \epsilon_{k}\} \\ \label{continuum_limit_inf_2} & = & \inf \{r > 0 : \mathbf{P}(\mathcal{A}_{[n,\rho]}^{r}) \geq \epsilon_{0}\} \\ & = & r(n,\rho,\epsilon_{0}) \nonumber \\ & = & r_{0} \nonumber
\end{eqnarray}
where eq. $(\ref{continuum_limit_inf_1})$ and eq. $(\ref{continuum_limit_inf_2})$ follow since $r_{k} \in \{r > 0 : \mathbf{P}(\mathcal{A}_{[n,\rho]}^{r}) \geq \epsilon_{k}\} \bigcap \{r > 0 : \mathbf{P}(\mathcal{A}_{[n,\rho]}^{r}) \geq \epsilon_{0}\}$ for all $k \geq 1$ and $\epsilon_{k} \rightarrow \epsilon_{0}$ as $k \rightarrow \infty$.  
\end{proof}

\end{document}